\documentclass{article}

\usepackage[margin=1in,footskip=0.25in]{geometry}

\usepackage[hyphens]{url}            

\usepackage[utf8]{inputenc} 
\usepackage[T1]{fontenc}    
\usepackage{hyperref}       
\usepackage{amsfonts}       
\usepackage{microtype}      
\usepackage{xcolor}         
\usepackage{natbib}
\usepackage{comment}
\usepackage{tikz}
\usetikzlibrary{positioning, calc, decorations.pathreplacing}

\def\showauthornotes{0}
\usepackage{amsmath,amssymb}
\usepackage{algorithm}
\usepackage{algpseudocode}
\usepackage{cryptocode}
\usepackage{float}
\usepackage{amsthm}
\usepackage{bbm}
\newtheorem{theorem}{Theorem}[section]
\newtheorem*{theorem*}{Theorem}

\newtheorem*{proposition*}{Proposition}
\newtheorem{lemma}[theorem]{Lemma}
\newtheorem*{lemma*}{Lemma}

\newtheorem*{conjecture*}{Conjecture}

\newtheorem*{fact*}{Fact}

\newtheorem*{hypothesis*}{Hypothesis}

\theoremstyle{definition}
\newtheorem{definition}[theorem]{Definition}
\newtheorem*{definition*}{Definition}

\newtheorem*{problem*}{Problem}

\theoremstyle{remark}

\newtheorem*{claim*}{Claim}

\newtheorem*{remark*}{Remark}
\newtheorem{observation}[theorem]{Observation}
\newtheorem*{observation*}{Observation}

\usepackage{prettyref}
\newcommand{\pref}{\prettyref}

\newcommand{\savehyperref}[2]{\texorpdfstring{\hyperref[#1]{#2}}{#2}}

\newrefformat{eq}{\savehyperref{#1}{\textup{(\ref*{#1})}}}
\newrefformat{prog}{\savehyperref{#1}{\textup{Program \ref*{#1}}}}
\newrefformat{eqn}{\savehyperref{#1}{\textup{(\ref*{#1})}}}
\newrefformat{lem}{\savehyperref{#1}{Lemma~\ref*{#1}}}
\newrefformat{def}{\savehyperref{#1}{Definition~\ref*{#1}}}
\newrefformat{thm}{\savehyperref{#1}{Theorem~\ref*{#1}}}
\newrefformat{cor}{\savehyperref{#1}{Corollary~\ref*{#1}}}
\newrefformat{cha}{\savehyperref{#1}{Chapter~\ref*{#1}}}
\newrefformat{sec}{\savehyperref{#1}{Section~\ref*{#1}}}
\newrefformat{app}{\savehyperref{#1}{Appendix~\ref*{#1}}}
\newrefformat{tab}{\savehyperref{#1}{Table~\ref*{#1}}}
\newrefformat{fig}{\savehyperref{#1}{Figure~\ref*{#1}}}
\newrefformat{hyp}{\savehyperref{#1}{Hypothesis~\ref*{#1}}}
\newrefformat{alg}{\savehyperref{#1}{Algorithm~\ref*{#1}}}
\newrefformat{rem}{\savehyperref{#1}{Remark~\ref*{#1}}}
\newrefformat{item}{\savehyperref{#1}{Item~\ref*{#1}}}
\newrefformat{step}{\savehyperref{#1}{step~\ref*{#1}}}
\newrefformat{conj}{\savehyperref{#1}{Conjecture~\ref*{#1}}}
\newrefformat{fact}{\savehyperref{#1}{Fact~\ref*{#1}}}
\newrefformat{prop}{\savehyperref{#1}{Proposition~\ref*{#1}}}
\newrefformat{prob}{\savehyperref{#1}{Problem~\ref*{#1}}}
\newrefformat{claim}{\savehyperref{#1}{Claim~\ref*{#1}}}
\newrefformat{relax}{\savehyperref{#1}{Relaxation~\ref*{#1}}}
\newrefformat{red}{\savehyperref{#1}{Reduction~\ref*{#1}}}
\newrefformat{part}{\savehyperref{#1}{Part~\ref*{#1}}}
\newrefformat{obs}{\savehyperref{#1}{Observation~\ref*{#1}}}
\newrefformat{corr}{\savehyperref{#1}{Corollary~\ref*{#1}}}

\ifnum\showauthornotes=1
\newcommand{\Authornote}[2]{{\sffamily\small\color{red}{[#1: #2]}}}
\newcommand{\Authornotecolored}[3]{{\sffamily\small\color{#1}{[#2: #3]}}}
\newcommand{\Authorcomment}[2]{{\sffamily\small\color{gray}{[#1: #2]}}}
\newcommand{\Authorstartcomment}[1]{\sffamily\small\color{gray}[#1: }

\newcommand{\Authorfnote}[2]{\footnote{\color{red}{#1: #2}}}
\newcommand{\Authorfixme}[1]{\Authornote{#1}{\textbf{??}}}
\newcommand{\Authormarginmark}[1]{\marginpar{\textcolor{red}{\fbox{\Large #1:!}}}}
\else
\newcommand{\Authornote}[2]{}
\newcommand{\Authornotecolored}[3]{}
\newcommand{\Authorcomment}[2]{}
\newcommand{\Authorstartcomment}[1]{}

\newcommand{\Authorfnote}[2]{}
\newcommand{\Authorfixme}[1]{}
\newcommand{\Authormarginmark}[1]{}
\fi

\newcommand{\Brac}[1]{\left[#1\right]}

\newcommand{\Abs}[1]{\left\lvert#1\right\rvert}

\newcommand{\norm}[1]{\lVert#1\rVert}

\newcommand{\iprod}[1]{\langle#1\rangle}

\newcommand{\Esymb}{\mathbb{E}}
\newcommand{\Psymb}{\mathbb{P}}

\DeclareMathOperator*{\E}{\Esymb}

\DeclareMathOperator*{\ProbOp}{\Psymb}

\renewcommand{\Pr}{\ProbOp}

\newcommand{\Ex}[2][]{\E_{{#1}}\Brac{#2}}

\DeclareMathOperator{\sign}{sign}

\newcommand{\R}{\mathbb R}

\newcommand{\cA}{\mathcal A}
\newcommand{\cB}{\mathcal B}

\newcommand{\cF}{\mathcal F}
\newcommand{\cG}{\mathcal G}
\newcommand{\cH}{\mathcal H}

\newcommand{\eps}{\epsilon}

\newcommand{\horacle}{\mathcal{O}}

\newcommand{\NTIMEo}{\textsf{NTIME}^{\horacle}}

\newcommand{\coNTIMEo}{\textsf{coNTIME}^{\horacle}}
\newcommand{\clso}{\Gamma^{\horacle}}

\newcommand{\func}{\mathcal{F}}
\newcommand{\vx}{\boldsymbol{x}}

\newcommand{\rinit}{r_{\textrm{init}}}
\DeclareMathOperator{\bern}{\mathcal{B}}
\newcommand{\1}{\boldsymbol{1}}

\DeclareMathOperator{\Adv}{Adv}

\newcommand{\pq}{p}

\newcommand{\Jnote}{\Authornotecolored{red}{J}}

\newcommand{\Znote}{\Authornotecolored{cyan}{Zhiyang}}

\title{Avoiding Obfuscation with Prover-Estimator Debate}

\author{
  Jonah Brown-Cohen\\
  Google DeepMind\\
  \texttt{jonahbc@google.com}
  \and
  \setcounter{footnote}{1}
  Geoffrey Irving\thanks{Work completed while at UK AI Security Institute.} \\
  Resolution\\
  \texttt{irving@resolution.org}
  \and
  Georgios Piliouras\\
  Google DeepMind\\
  \texttt{gpil@google.com}\\
  \and
  Lijie Chen\\
  UC Berkeley\\
  \texttt{lijiechen@berkeley.edu}
  \and
  Jiawei Li\\
  UT Austin\\
  \texttt{davidlee@cs.utexas.edu}
  \and
  Zhiyang Xun\\
  UT Austin\\
  \texttt{zxun@cs.utexas.edu}
}

\begin{document}

\maketitle

\begin{abstract}
Training powerful AI systems to exhibit desired behaviors hinges on the ability to provide accurate human supervision on increasingly complex tasks. A promising approach to this problem is to amplify human judgement by leveraging the power of two competing AIs in a debate about the correct solution to a given problem. Prior theoretical work has provided a complexity-theoretic formalization of AI debate, and posed the problem of designing protocols for AI debate that guarantee the correctness of human judgements for as complex a class of problems as possible. Recursive debates, in which debaters decompose a complex problem into simpler subproblems, hold promise for growing the class of problems that can be accurately judged in a debate. However, existing protocols for recursive debate run into the \emph{obfuscated arguments problem}: a dishonest debater can use a computationally efficient strategy that forces an honest opponent to solve a computationally intractable problem to win.
We mitigate this problem with a new recursive debate protocol that, under certain stability assumptions, ensures that an honest debater can win with a strategy requiring computational efficiency comparable to their opponent.
\end{abstract}
\begin{figure}[ht!]
\centering
\begin{tikzpicture}[
    node distance=1.5cm,
    a_node/.style={circle, draw, fill=blue!30, minimum size=0.9cm},
    x_node/.style={circle, draw, fill=gray!30, minimum size=0.9cm},
    b_node/.style={rectangle, draw, fill=red!30, minimum size=0.9cm},
    thick,
    scale=0.85, transform shape 
]

\node[anchor=south] at (-3, 5) {\textbf{Original recursive debate}};

\node[x_node] (x1) at (-3, 4) {$x$};

\node[a_node] (y11) at (-5, 2) {$y_1$};
\node[a_node] (y12) at (-3, 2) {$y_2$};
\node at (-1.5, 2) {$\cdots$};
\node[a_node] (y1m) at (0, 2) {$y_q$};
\node[text=orange!90!black] at (0, 2.7) {flaw!};

\node[text=blue!70!black, align=center] at (-5.2, 3.3) {$A$ makes\\subclaims};

\draw (x1) -- (y11);
\draw (x1) -- (y12);
\draw (x1) -- (y1m);

\draw[->] (y12) -- +(0,-2)
node[midway, right, text=red!70!black, align=left] {$B$ chooses\\a subclaim};

\node[anchor=south] at (7, 5) {\textbf{Prover-estimator debate}};

\node[x_node] (x2) at (7, 4) {$x$};

\node[a_node] (y21) at (5, 2) {$y_1$};
\node[a_node] (y22) at (7, 2) {$y_2$};
\node at (8.5, 2) {$\cdots$};
\node[a_node] (y2m) at (10, 2) {$y_q$};
\node[text=orange!90!black] at (10, 2.7) {flaw!};

\node[text=blue!70!black, align=center] at (5.0, 3.3) {$A$ makes\\subclaims};

\draw (x2) -- (y21);
\draw (x2) -- (y22);
\draw (x2) -- (y2m);

\draw[->, thick, >=stealth, black] (1.5, 2) -- (3.5, 2)
node[midway, above] {Our paper};

\node[b_node] (p1) at (5, 0.4) {$p_1$};
\node[b_node] (p2) at (7, 0.4) {$p_2$};
\node at (8.5, 0.4) {$\cdots$};
\node[b_node] (pm) at (10, 0.4) {$p_q$};

\draw (y21) -- (p1);
\draw (y22) -- (p2);
\draw (y2m) -- (pm);

\node[text=red!70!black, align=center] at (8.5, 1.4) {$B$ assigns\\probabilities};

\draw[->] (p2) -- +(0,-2.3)
node[midway, right, text=blue!70!black, align=left] {$A$ chooses a subclaim\\and claims $B$ is wrong\\in a specific direction};

\end{tikzpicture}
\caption{The original recursive debate protocol suffered from the \textit{obfuscated arguments problem}: debater $A$ could decompose an easy question $x$ into hard subclaims $y_1, y_2, \ldots, y_q$, and debater $B$ would fail to find the flaw even if he knew one existed. In prover-estimator debate, $B$ assigns probabilities to subclaims and $A$ chooses a probability to claim that $B$ is wrong in a specific direction. Since $A$ must point to a flaw in $B$'s probabilities, $B$ wins if neither player can locate a flaw.}
\label{fig:debate-comparison}
\end{figure}
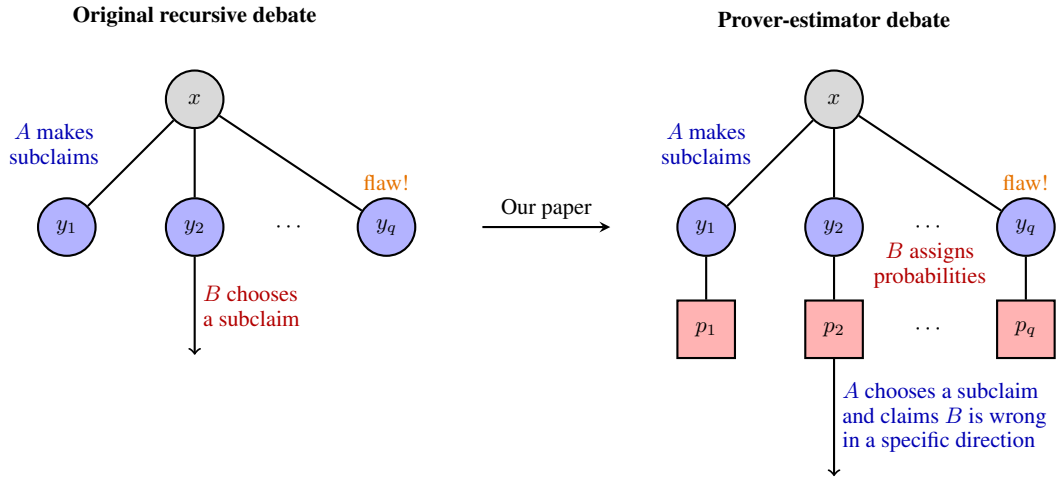
\section{Introduction}

Modern machine learning (ML) is fundamentally driven by, and further enables, fast and complex computations. We leverage AI precisely because it allows us to automate and optimize tasks with an efficiency that surpasses human capabilities. This rapid computation, permeating increasingly high-complexity and intelligence tasks, is the engine of progress in countless domains. However, this very strength creates a fundamental tension: the speed and complexity of these computations make exhaustive inspection impossible. Understanding and verifying the internal workings of these powerful AI systems---what's going on ``behind the scenes''---becomes increasingly intractable. This tension directly leads to the critical problem of \emph{scalable oversight}: how to provide accurate training supervision, such that these powerful, yet opaque, AI systems behave as intended, even when their tasks are too complex for direct human understanding and verification?

Ensuring the reliable and safe operation of advanced AI systems is a central challenge for the field. The potential for unintended consequences arising from misaligned or poorly understood AI behavior requires the development of robust oversight mechanisms in order to provide a training reward signal that accurately reflects human values and intentions. Effective scalable oversight is not merely a desirable feature; it is a foundational requirement for building trustworthy AI that aligns with human intentions and avoids producing systematically undesirable outcomes. Progress in this area is crucial for realizing the full potential of AI while mitigating potential risks.

Various approaches to scalable oversight have been proposed, including amplification \citep{christiano2018supervising}, recursive reward modeling \citep{leike2018scalable}, and AI debate \citep{irving2018ai}. All of these approaches are based on the intuition that it is possible to amplify the reward signal coming from human judgment by leveraging improving AI capabilities. For the case of AI debate, \cite{irving2018ai} formalized this intuition in the language of computational complexity theory: they proved that a computationally limited human can accurately judge the correct solution to problems of much larger computational complexity by observing a specifically structured debate between two computationally powerful AIs.
The main idea is that the two powerful AI debaters recursively break down a complex problem into simpler sub-problems, eventually terminating in subproblems that are simple enough for the computationally limited human to judge directly.

As an example, consider the case of training a software-engineering AI system to produce large, complex pull requests. Rather than have a human expert carefully review each pull request during training, one could train the AI system via debate, where the first debater produces the pull request, and the second debater acts as a code reviewer attempting to point out flaws in the produced code. For example, the code-reviewing debater might claim that a particular function returns an incorrect value on a certain input. The coding debater could defend their code by tracing through the function on that input and arguing that the correct value is returned. The code-reviewing debater might then zero-in on a single step of the trace calling a library function, and claim that the library function does not behave as the coding debater has claimed.
The coding debater might appeal to the library's documentation to justify their claim, and the code-reviewing debater might respond by claiming that the version of the library in the production environment differs from that documentation.
At this point, a human judge can quickly check which of the two debaters is correct about the production library version, without having to read and fully understand the complex pull request.
Hence, the human judge only has to understand a specific, narrow claim about one piece of the code to provide an accurate reward signal.

However, a significant vulnerability exists in many recursive approaches, including recursive debate: the \emph{obfuscated arguments problem} \citep{barnes2020obfuscated}. In recursive debate, a dishonest debater can \emph{strategically} decompose an easy problem into difficult, even intractable, subproblems. This renders an honest debater unable to find a flaw in their opponent’s argument even if both debaters know that a flaw exists somewhere. The debate will thus fail to identify the correct answer.   Obfuscation in recursive debate is not merely a theoretical concern; it emerged naturally in prior empirical work \citep{barnes2020obfuscated}, suggesting obfuscation is a likely \emph{attractor state} in some settings. 

\cite{brown-cohen24scalable} avoid the issue of obfuscation in debate by doing away with recursion entirely. In this case, an efficient honest debater can always win, but debaters are required to write-out the full human-judgeable argument for the correctness of their claims. To return to the software engineering example, this would mean that the coding debater has to provide complete, detailed documentation of how and why every single piece of the code works, such that any plausible critique by the code-reviewing debater is already answered by the documentation. This is a significant sacrifice relative to recursive debate, which allows the debaters to efficiently and adaptively zoom-in on a small fraction of the full, detailed human-judgeable argument.

\paragraph{Our Approach.} We introduce \emph{prover-estimator debate}, a protocol for AI debate that mitigates the obfuscated arguments problem while retaining the efficiency of recursive debate. In prover-estimator debate, the prover (Alice) decomposes a problem into subclaims. The estimator (Bob) has the \emph{sole} task of assigning probabilities to the prover's subclaims. The prover uses these probabilities to decide which subclaim to recursively debate further. This departs from previous debate protocols, where the opponent selects which subclaim to challenge. We show that Bob can always produce probabilities that are \emph{indistinguishable from the truth by Alice}, while expending computational effort not too much greater than Alice. Alice cannot therefore (on average) select a subproblem where Bob is non-trivially wrong. However, prover-estimator debate has limitations. Alice’s task becomes harder because Bob can defeat an argument by assigning low probabilities to many subclaims. This rules out certain delicate, yet correct arguments that are highly sensitive to the probabilities assigned to subproblems. We formalize this condition by introducing a \emph{stability} requirement, 
\emph{$\rho$-stability}, which demands that the validity of arguments cannot hinge on arbitrarily small changes in subclaim probabilities\footnote{An earlier preprint had an error in the main theorem due to an issue with the way stability was defined. At a high level, the corrected definition bounds changes in  marginal probabilities of subclaims, while the original mistaken definition used conditional probabilities. This error was found and fixed by the last three authors.}. We require stability only for usefulness, not safety: even if stability does not hold, an honest estimator can still win against any incorrect argument by a dishonest prover.

To gain intuition for how prover-estimator debate works, let us return to the pull request example.
Anyone who has undertaken code review knows that reading code and directly identifying bugs is a difficult task. However, it is often possible to estimate where bugs are likely to occur from the structure of the code, and then to request additional evidence in the form of tests that demonstrate the desired behavior.
Prover-estimator debate in this setting requires Alice to provide reliable test coverage, so that even when Bob is uncertain that any single test provides sufficient evidence that the code is correct, taken together it is still likely that Alice's code exhibits the desired behavior.
On the other hand, if Alice has engaged in obfuscation by producing code that fails in a way neither Alice nor Bob can readily identify, then Bob can output low probabilities for Alice's claims, and Alice will not be able to reliably identify any probability as being unreasonable.

\paragraph{Our Results.} We provide rigorous theoretical guarantees demonstrating that prover-estimator debate, \emph{combined with the stability requirement}, mitigates the obfuscated arguments problem. Specifically, under reasonable assumptions about the structure of the problems being debated (formalized as \emph{recursively decomposable problems} and the $\rho$-stability condition), we show:

\begin{enumerate}
    \item \textbf{Completeness:} An honest prover, who correctly provides a stable decomposition into solvable subproblems, can always win the debate against any dishonest estimator.
    \item \textbf{Soundness:} An honest estimator can win on average against a dishonest prover, \emph{even if stability is violated and the subproblems are obfuscated}, using computational effort comparable to their dishonest opponent. 
\end{enumerate}

The requirement that it is possible for the prover to provide a stable decomposition is the primary restriction on the class of problems that can be accurately judged by our protocol. That is, there can be problems where the prover can efficiently decompose the original problem into solvable subproblems, but all such efficient decompositions are unstable. In this case, our protocol errs on the side of safety, and rejects the prover's argument. This is necessary for the protocol to succeed, because it cannot distinguish between true and false arguments made via unstable decompositions. Note that all that is required is that the prover is able to find a \emph{single} stable decomposition to use in the debate. Even if other, unstable decompositions exist, the prover will lose on average when using them, and thus will be incentivized via debate training to select the winning, stable decomposition. Intuitively, the stability condition can be seen as requiring that small changes to marginal probabilities of subclaims should not have large effects on the truth of the original claim. We argue that this is a natural requirement when dealing with fuzzy human judgments, where one should be fundamentally suspicious of arguments that depend very precisely on the composition of many inherently imprecise subclaims.

Prover-estimator debate provides a step towards building robust and trustworthy AI systems capable of tackling complex tasks beyond direct human oversight. Our formal notion of argument stability, in particular, clarifies the conditions required to avoid the obfuscated arguments problem and enables more targeted research on overcoming this problem. Our approach also suggests general lessons for designing future debate protocols, including the value of asymmetric debate protocols and explicit uncertainty estimates. This work opens up new avenues for research, both theoretical and empirical, in the pursuit of safe and aligned AI.

\section{Related work}
The prior theoretical work on the design of protocols for AI debate consists primarily of \cite{irving2018ai}, which introduced AI debate and its complexity-theoretic formalization, and \cite{brown-cohen24scalable} which introduced doubly-efficient debate. As per the introduction, the original debate proposal makes the too-strong assumption that both debaters are computationally unbounded. Doubly-efficient debate relaxes this assumption, at the price of solving a more limited class of problems. In a complementary direction, the work of \cite{chen2023playing} focuses on the theory of learning-in-games for AI debate. The objective in this case is to design efficient learning algorithms for debate-inspired games where the space of strategies is exponentially large, but the outcome of the game can be efficiently judged given strategies of two players.

The much earlier work in complexity theory by \cite{chandra1981alternation} introduced, with different terminology, the question of studying the complexity of polynomial-time verifiable arguments between computationally unbounded provers. 
Later \cite{feige1997making} proved stronger results in a complexity-theoretic model of an efficient verifier that judged arguments between competing computationally unbounded provers.
A similar model with competing provers was studied in the context of delegation of computation to untrusted servers in \cite{canetti2013refereed}.

There has recently been a line of research applying single-prover interactive proofs to machine learning for proving correctness of model outputs \citep{amit2024models, anil2021learning,hammondneural}, and for improving interpretability and legibility \citep{waldchen2022merlin, kirchner2024prover}.
Empirical work on debate has focused primarily on debates where both the debaters and the judge are large language models (LLMs), and has shown some positive signs that debate allows for more accurate judgements on standard datasets \citep{michael2023debate,kentonscalable,khan2024debating}.

\section{Preliminaries}
We use the notation $[n] = \{1,\dots,n\}$.
For a sequence of elements $x_1,\dots, x_n$ from a set $X$ and $S\subseteq [n]$, we use the notation $x_S$ for the subsequence of $x_i$ with $i \in S$. Similarly, we write $x_{\leq i}$, $x_{<i}$, and $x_{\neq i}$ to denote the subsequence consisting of $x_j$ where $j \leq i, j < i$, and $j\neq i$ respectively.
For a probability $p \in [0,1]$ we use the notation $\bern(p)$ to denote the Bernoulli random variable that takes value $1$ with probability $p$, and $0$ otherwise. For a finite set $Z$, let $\Delta_Z$ denote the set of all probability distributions over $Z$, i.e., $\Delta_Z = \{p : Z \to [0,1] \mid \sum_{z \in Z} p(z) = 1\}$.

\subsection{Complexity-theoretic debate}
As in prior work on debate, we use computational complexity theory to formalize the power of the debaters and human judges, as well as the complexity of the problems that can be judged via debate.

\paragraph{Algorithms and oracles.} Algorithms are modeled by Turing machines $M:\{0,1\}^*\to\{0,1\}$ which take an input $x$ and output a yes-or-no answer encoded as a bit $M(x)\in\{0,1\}$. Furthermore, we consider \emph{oracle Turing machines} that additionally have black-box access to an oracle $\horacle:\{0,1\}^{n} \to \{0,1\}$. Once a query $u \in \{0,1\}^{n}$ is written on its tape, an oracle Turing machine can receive the oracle answer $\horacle(u) \in \{0,1\}$ in a single step.
We use the notation $M^\horacle$ to denote an oracle Turing machine with black-box query access to the oracle $\horacle$.
In the setting of AI debate, the Turing machine $M$ corresponds to a natural language plan or set of instructions to follow, and queries to the oracle $\horacle(u)$ corresponds to querying human judgement on a particular question $u$.

\paragraph{Classes of problems.} The goal in debate is to design protocols that allow us to solve a broad class of computational problems, while requiring few queries to the human judgement oracle $\horacle$.
A computational problem with a yes-or-no answer is encoded by a \emph{language} $L\subset \{0,1\}^*$ consisting of all the inputs $x$ for which the answer is ``yes''. We abuse notation to view $L$ as a function $L:\{0,1\}^*\to \{0,1\}$ where $L(x)=1$ if $x\in L$ and $L(x)=0$ otherwise. 
The primary class of problems that we consider is $\NTIMEo(T)$, the set of problems where for every instance the answer is ``yes'' if and only if there is proof that the answer is ``yes'', and this proof can be verified in time $T$ with access to the human judgement $\horacle$.
\begin{definition}
    A language $L$ is in $\NTIMEo(T)$ if there is a time $T$ oracle Turing machine $M$ such that $x\in L$ if and only if there exists a proof $y$ of length at most $T$ such that $M^{\horacle}(x,y) = 1$.
\end{definition}
We also use the class $\coNTIMEo(T)$ of languages $L$ whose complement is in $\NTIMEo(T)$. $\coNTIMEo(T)$ is the set of yes-or-no problems with efficiently checkable proofs of ``no''.
\begin{definition}
    A language $L$ is in $\coNTIMEo(T)$ if there is a time $T$ oracle Turing machine $M$ such that $x\notin L$ if and only if there exists a proof $y$ of length at most $T$ such that $M^{\horacle}(x,y) = 1$.
\end{definition}
We use the notation $\clso(T) = \NTIMEo(T)\cap\coNTIMEo(T)$ for the set of problems where one can verify both the cases $x\in L$ and $x \notin L$ with a proof that is checkable in time $T$ given access to the ground-truth oracle. Our main results focus on this class.

Our debate protocols will rely on recursively decomposing the execution of a verifier machine $M$ into individual steps. To formalize this we introduce the notion of the \emph{transcript} of a machine $M$.
\begin{definition}
   The \emph{transcript} of a time $T$ machine $M$ on input $x$ is a string $\tau \in \{0,1\}^T$ where $\tau_i$ is the bit written at the current head position of $M$ at time step $i$.
\end{definition}

\subsection{Game theory}
The main game-theoretic concept used in our results is a \emph{Stackelberg equilibrium}. This equilibrium concept models a situation where one player (the leader) must commit to a strategy, after which the other player (the follower) chooses their strategy. A Stackelberg equilibrium is a strategy for the leader that achieves the best possible payoff, given that the follower plays a best response.

\begin{definition}[Stackelberg equilibrium]
Consider a two-player (leader and follower) game with a sequential structure, where
$S_L$ is the set of strategies for the leader, $S_F$ is the set of strategies for the follower,
$u_L(s_L, s_F)$ is the payoff for the leader, and $u_F(s_L, s_F)$ is the payoff for the follower.
The game proceeds as follows: first the leader chooses a strategy $s_L \in S_L$,
then the follower observes $s_L$ and chooses a strategy $s_F \in S_F$.

An \emph{$\alpha$-approximate Stackelberg equilibrium} is a strategy profile $(s_L^*, s_F^*(s_L))$ with

\begin{enumerate}
    \item \textit{Follower's best response:} For every leader strategy $s_L$, the follower's strategy $s_F^*(s_L)$ is an approximate best response:
    \begin{equation}
    u_F(s_L, s_F^*(s_L)) \geq u_F(s_L, s_F) - \alpha \quad \forall s_F \in S_F.
    \end{equation}

    \item \textit{Leader's optimal strategy:} The leader's strategy $s_L^*$ approximately maximizes the leader's payoff, \emph{given} the follower's best response function:
    \begin{equation}
    u_L(s_L^*, s_F^*(s_L^*)) \geq u_L(s_L, s_F^*(s_L)) - \alpha \quad \forall s_L \in S_L.
    \end{equation}
\end{enumerate}
\end{definition}

Some intuition is useful for why Alice-leading, Bob-following Stackelberg equilibria are the right setting. For completeness, we will use a prover Alice that always gives correct answers: this will win against any Bob regardless of Bob's computational complexity, and thus works as an Alice-leading strategy. For soundness, however, we must find an estimator Bob that uses computation not too much larger than Alice, which we do by repeatedly playing Alice against herself. This method of Bob construction works only once Alice is fixed, and thus requires Bob to be the Stackelberg follower.

\section{Debate protocols and games}
\label{sec:debate-games}
The goal of debate is to provide an accurate training signal on questions too hard for humans to judge directly. Training with debate involves asking an AI $A$ to give the answer to a question $x$ drawn from a distribution $\mu$, and having another AI $B$ debate the correctness of $A$'s answer. At the end of the debate, human judgment is used to provide a reward signal to both AI debaters involved.
Thus, the more specific goal is to design a \emph{debate protocol}, i.e.\ the rules of the debate and the rewards, such that training $A$ and $B$ with the reward signal provided by the debate protocol leads to $A$ giving the correct answer to $x$.
This can be formalized game theoretically: a debate protocol induces a game between the AI debaters, and the protocol succeeds if every equilibrium of the game leads $A$ to correctly answer almost all questions $x \sim \mu$.

We begin with the formal definition of a debate protocol, which specifies two sets of possible strategies as well as the verification procedure to determine the rules and outcome of a debate.

\begin{definition}[Debate protocol]
 A \emph{debate protocol} is given by a family of tuples $(\cA_n,\cB_n,V)$, where for each natural number $n$, $\cA_n$ and $\cB_n$ are finite sets of \emph{strategies} defined by functions $A,B :\{0,1\}^{O(n)} \to \{0,1\}^{O(n)}$, and the verifier $V:\{0,1\}^*\to[-1,1]$ is a polynomial time oracle Turing machine. 
\end{definition}

A debate protocol induces a zero-sum game where debaters select strategies from $\cA_n$ and $\cB_n$, exchange messages, and the verifier determines a final payoff based on all the messages in the debate.

\begin{definition}[Two-player, zero-sum debate game]
 Given a language $L$ and a distribution $\mu$ on inputs $x\in \{0,1\}^n$, a debate protocol induces a two-player zero-sum game where one player selects a strategy $A\in\cA_n$, the other a strategy $B\in\cB_n$. The payoffs are defined by the following process:
 \begin{enumerate}
    \item Sample a random $x\sim \mu$.
    \item The $A$ player outputs a bit $A(x) \in \{0,1\}$ to claim that $L(x) = A(x)$ (that is, $A$ suggests whether the claim is true or false).
    \item The players $A$ and $B$ interact by exchanging messages on input $x$. In particular, $A$ and $B$ generate a sequence of messages $a_i = A(x,a_1,b_1,\dots,a_{i-1},b_{i-1})$ and $b_i = B(x,a_1,b_1,\dots,a_{i-1})$.
    \item The verifier $V$ runs on the interaction to determine the payoffs. 
    \item The $A$ player receives payoff $V^{A,B}(x)$ and the $B$ player receives payoff $-V^{A,B}(x)$.
 \end{enumerate}
\end{definition}

The rules of the debate are encoded by the verifier $V$, while the sets of functions $\cA_n,\cB_n$ encode the possible strategies that debaters $A$ and $B$ can employ. For example, $V$ might be a set of prompts and scaffolding for an LLM debate, while $\cA_n$ and $\cB_n$ are sets of possible LLMs of given architecture and size.
When the input length $n$ is clear we write $\cA$ and $\cB$ for the strategy sets.
Our setting in this paper is when $\cA$ and $\cB$ both contain functions of bounded computational complexity. 
We measure complexity purely relatively, i.e.\ we show that it is possible to incentivize the $A$-player to provide the correct solution when the set of strategies $\cB$ are all low-complexity functions of strategies from $\cA$. Each debate will then also have low $\cA$-complexity, as each round of the debate requires a constant number of queries to the strategies of $A$ and $B$.

\section{Recursively decomposable problems}
\label{sec:decomposable-problems}
In this section we introduce the problem class for which our debate protocols will succeed. The intuition is that for any $L \in \clso(T)$, checking the proof $y$ that $x\in L$ or $x\notin L$ can be recursively decomposed into subproblems. Given a transcript $\tau$ of a verifier machine $M(x,y)$, a natural decomposition is to split into the two subproblems of checking that the first half of $\tau$ is correct and that the second half is correct. This decomposition can be repeated $\log T$ times until the final level, which checks individual steps in $\tau$. We formalize the notion of a recursive decomposition as follows.
\begin{definition}
    \label{def:recursive-decomposition}
    Let $L \in \clso(T)$ be a language and $M$ the machine that verifies proofs for membership of $x\in L$ (or $x \notin L$ respectively).
    A depth $d$, width $q$ \emph{recursive decomposition} for $L$ is defined by a time $O(nq)$ machine $M_D$ that takes as input an original input $x\in \{0,1\}^n$, a sequence of $q$ \emph{queries} $y = (y_1,\dots,y_q)\in\{0,1\}^{n\times q}$, and a sequence of $q$ \emph{query answers} $z = (z_1,\dots,z_q) \in \{0,1\}^q$.
    For any $x \in L$ there exists an \emph{implicit proof} $u\in\{0,1\}^m$ with $m \leq q^d$ such that for all $k \in [d]$, the correct answers $z^k$ to queries $y^k$ are defined recursively as follows. 
    \begin{enumerate}
        \item For the base case $k = 0$, for all leaves $i$ there exists a subset $S_i \subseteq [m]$ of coordinates in the implicit proof, such that the query $y_i^0 = S_i$, and the correct answer $z_i^0 = \horacle(u_{S_i})$.
        \item The correct answers to queries $y^{k-1}$ satisfy $z^k_i = M_D(x,y^{k-1}_i,z^{k-1})$, where $z^{k-1}$ are the correct answers to $y^{k-1}$. That is, answers at $k-1$ support answers at $k$.
        \item At the top level $k = d$ of the recursion, $M_D(x,y^d,z^d) = M^{\horacle}(x,y^*)$ where $y^*$ is a witness to the original machine $M^{\horacle}$ for $L$.
    \end{enumerate}
\end{definition}
Intuitively a recursive decomposition breaks down the task of checking that $x\in L$ into $q$ subproblems, which are themselves broken down into $q$ subproblems, with the recursion bottoming out in ground-truth oracle queries to an implicit proof $u$.
Thus, the queries $y_i^k$ at non-leaf steps in the recursion have correct answers $z_i^k$ that are functions of larger subsets of the $u$, with the top level depending on the entire implicit proof.
For every language $L\in\clso(T)$ there is a natural recursive decomposition of depth $d=\log T$ and width $q=2$, where the implicit proof $u$ is just the transcript $\tau$ of $M(x,y^*)$. The base case queries $y_i^0$ are just individual steps in the transcript $\tau$, the queries $y_i^1$ are claims that $2$ consecutive steps of $\tau$ are performed correctly, and in general $y_i^k$ is the claim that $2^k$ consecutive steps of $\tau$ are correct, which is clearly the AND of two queries $y_1^{k-1},y_2^{k-1}$ regarding sequences of $2^{k-1}$ steps. Hence, $M_D$ simply computes the AND function on the query answers $z_1$ and $z_2$.

For a general language $L\in\clso(T)$ some of the subproblems given by queries $y_{i}$ may be computationally intractable for debaters using only algorithms $A\in \cA$, while others can be efficiently solved by some $A\in\cA$. Our goal is to design a protocol that is agnostic to the difficulty of the subproblems. If all subproblems are easy it should be possible for the honest debater to win the debate without expanding the full proof tree corresponding to the recursive decomposition. 
At the same time, it should not be possible for a dishonest debater to win by intentionally decomposing an easy problem into hard subproblems, which the honest debater is unable to refute.
To achieve this goal, the computational complexity of the provers in the debate will be measured purely relatively, i.e.\ $\cB$ contains strategies that are low complexity functions of the strategies in $\cA$.

We next capture problems where an AI can solve any subproblem in the recursive decomposition if asked, and thus in some sense ``knows'' the correct answer. The definition requires that there is a strategy $A\in\cA$ that correctly answers all queries that appear in the recursive decomposition $M_D$.

\begin{definition}[$\cA$-provable]
\label{def:A-provable}
For each natural number $n$, let $\cA_n$ be a set of functions on $\{0,1\}^n$ and let $0 < \gamma < 1$. Let $\mu$ be a distribution on inputs $x$ of length $n$, let $L\in\clso(T)$ be a language, and let $M_D$ be a recursive decomposition of $L$. The decomposition $M_D$ is $\cA$-provable with probability $1-\gamma$ if there exists a single $A\in\cA$ such that with probability at least $1-\gamma$ over $x \sim \mu$:
\begin{enumerate}
\item In the base case of the decomposition, $A(x,j) = u_{j}$ for any $j\in S_i$, where $S_i$ is the subset given by $y_i^0$.
\item For all queries $y_i^k$ in the recursive decomposition, $A(x,y_i^k) = z_i^k$ is the correct query answer.
\item $A(x) = L(x)$.
\end{enumerate}
\end{definition}

We also need the arguments made in the debate to not depend too precisely on the probabilities of the answers to the subproblems.

\begin{definition}[$\rho$-stability]\label{def:epsilon-stable}
    For any $z  \in \{0,1\}^q$, let $D_{\epsilon}(z)$ be the set of distribution $\pq$ over $\{0, 1\}^q$ such that for any $i \in [q]$, $\Pr_{z' \sim \pq}[z'_i \neq z_i] \leq \epsilon$.

    Let $x\in\{0,1\}^n$ be an instance of a problem $L \in \clso(T)$ and let $M_D$ define a width $q$ recursive decomposition of $L$.
    $M_D$ is $\rho$-stable at node $(x,y)$ if for any $\epsilon > 0 $ and any distribution $\pq \in D_{\epsilon}(z)$, where $z$ is the ground truth solution to $(x,y)$, there is 
    \[\Pr_{z'\sim \pq}\left[M_D(x, y, z') \neq M_D(x, y, z)\right] \leq \rho \cdot \epsilon.\]

    Further, $M_D$ is $\rho$-stable for an input $x$ w.r.t. Alice's algorithm $A$ if for every node $(x,y)$ that might be visited in the debate tree generated by $A$, $M_D$ is $\rho$-stable at $(x,y)$.
\end{definition}

For simplicity, we may say $M_D$ is $\rho$-stable for an input $x$
when Alice's algorithm $A$ is clear from the context. Interestingly, this stability condition is captured by the notion of \emph{fractional block sensitivity} from boolean function analysis (see, e.g.,~\cite{KulkarniTal2016fbs}); this connection is further discussed in \pref{sec:stability-fbs}.

Intuitively \pref{def:epsilon-stable} says that changing the \emph{marginal} probabilities of the sequence $z$ by less than $\epsilon$ leads to small changes in the output of the verifier machine.
In practical settings where the truth is determined by aggregation of many pieces of uncertain evidence (e.g. empirical science, legal arguments), $\rho$-stability will hold for small constant values of $\rho$. For example, suppose that one wants to determine whether a majority of scientific literature supports the efficacy of a certain drug. If 2/3rds of studies show that the drug is effective, then even if up to 1/6th of all published studies turned out to have critical flaws, the majority outcome remains unchanged. In this case, the claim about the majority of scientific literature is $\rho$-stable for $\rho = 4$.
To see why, one can check that by symmetry, the way to maximize the chance of flipping the majority vote while minimizing the change in the marginal probability of overturning each individual study is to uniformly select 1/6th of the studies to flip among the 2/3 that say the drug is effective. In this case $\epsilon = \frac{1/6}{2/3} = 1/4$ and flipping these 1/6 of the studies flips the majority with probability $1 = 4\epsilon$. Thus the majority vote is $\rho$-stable with $\rho = 4$ in the case that 2/3rds of the inputs are set to true.

\section{The Prover-Estimator Debate Protocol}
\label{sec:two-player-protocol}

\begin{figure*}[ht!]
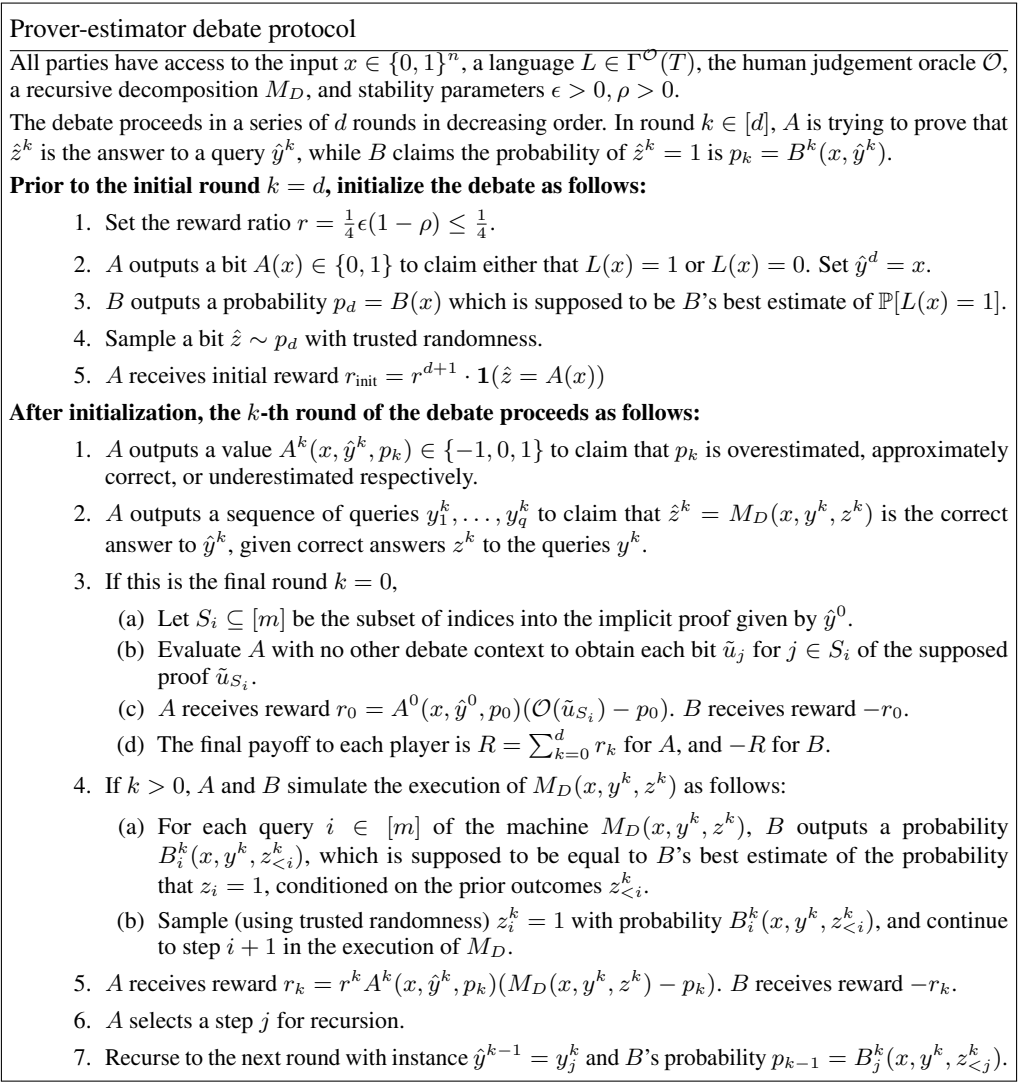

\begin{pchstack}[center,boxed]
      \procedure[mode=text,width=0.95\textwidth]{Prover-estimator debate protocol}{
      All parties have access to the input $x\in\{0,1\}^n$, a language $L\in\clso(T)$, the human judgment oracle $\horacle$, a recursive decomposition $M_D$, and stability parameters $\rho > 0$.\\
      The debate proceeds in a series of $d$ rounds in decreasing order.\\ 
      \textbf{Prior to the initial round $k=d$, initialize the debate as follows:}
      \begin{enumerate}
        \item Set the reward ratio $r = 1 / \rho$.
        \item $A$ outputs a bit $A(x) \in \{0,1\}$ to claim either that $L(x)=1$ or $L(x)=0$. Set $\hat{y}^d = x$.
        \item $B$ outputs a probability $\hat{p}_d = B(x)$ which is supposed to be $B$'s best estimate of $\Pr[L(x) = 1]$.
        \item Sample a bit $z \sim \hat p_d$ with trusted randomness.
        \item $A$ receives initial reward $\rinit = r^{d+1}\cdot\1({z} = A(x))$
      \end{enumerate}
      \textbf{After initialization, the $k$-th round of the debate proceeds as follows:}
        \begin{enumerate}
            \item $A$ outputs a value $A^k(x,\hat{y}^{k},\hat{p}_k)\in\{-1,0,1\}$ to claim that $\hat{p}_k$ is overestimated, approximately correct, or underestimated respectively.
            \item $A$ outputs a sequence of queries $y^k = (y^k_1,\dots,y^k_q)$ to claim that $\hat{z}^k \coloneqq M_D(x,y^k,z^k)$ is the correct answer to $\hat{y}^k$, given correct answers $z^k$ to the queries $y^k$.
            \item If this is the final round $k=0$, 
            \begin{enumerate}
                \item Let $S_i \subseteq [m]$ be the subset of indices into the implicit proof given by $\hat{y}^0$.
                \item Evaluate $A$ with no other debate context to obtain each bit $\tilde{u}_j$ for $j\in S_i$ of the supposed proof $\tilde{u}_{S_i}$. 
                \item $A$ receives reward $r_0 = A^0(x,\hat{y}^0,\hat{p}_0)(\horacle(\tilde{u}_{S_i}) -\hat{p}_0)$. $B$ receives reward $-r_0$.
                \item The final payoff to each player is $R = \sum_{k=0}^d r_k$ for $A$, and $-R$ for $B$.
            \end{enumerate}
            \item If $k > 0$, $B$ outputs a distribution $p_k \coloneqq B^k(x, y^k)$ over $\{0,1\}^q$ as its estimation for $z^k$.
            \item Sample $z^k \sim p_k$ with trusted randomness.
            \item $A$ receives reward $r_k = {r^{k}A^k(x,\hat{y}^k, \hat{p}_k)(M_D(x,y^k,z^k) - \hat{p}_k)}$. $B$ receives reward $-r_k$.
            \item $A$ selects a step $j$ for recursion.
            \item Recurse to the next round with instance $\hat{y}^{k-1} \coloneqq y^k_{j}$ and $B$'s probability $\hat{p}_{k-1} \coloneqq  \Pr_{z^k \sim p_k}[z^k_j = 1]$, i.e., $\hat{p}_{k-1}$ is the marginal probability of $z^k_j$ in $p_k$.
        \end{enumerate}
    }
\end{pchstack}
\caption{The prover-estimator debate protocol.}
\label{fig:protocol-cut-and-choose}
\end{figure*}

In this section we give our debate protocol and the main theorems that it correctly incentivizes debaters for problems in $\clso$.
The intuition is as in \pref{fig:debate-comparison}: debater $A$ decomposes the problem into subproblems, and debater $B$ gives probabilities for the answers to each subproblem. Debater $A$ then selects a subproblem for recursion where $B$'s probability is allegedly incorrectly estimated. Rewards are assigned in each round to $A$ for correctly predicting inconsistencies between $B$'s probabilities in the current round versus the previous round, with the available reward growing by a factor of $1/r$ in each round.
The final round uses the human judgment oracle $\horacle$ rather than $B$'s probabilities.

\subsection{The protocol}
Our protocol for two-player zero-sum debate is given in \pref{fig:protocol-cut-and-choose}. We can now state our main theorem.
\begin{theorem}
\label{thm:cut-and-choose}
Let $\cA$ be any set of functions on $\{0,1\}^n$ and $\eta > 0$. 
Let $\mu$ be a distribution on inputs $x\in\{0,1\}^n$, and $L \in \clso(n^c)$ be a language with a recursive decomposition $M_D$ of depth $d$ and width $q$. Let 0 < $\gamma < 1/2$ and take a stability parameter $\rho$.
Then the protocol of \pref{fig:protocol-cut-and-choose} run on $M_D$ satisfies:
 \begin{itemize}
     \item \textbf{Completeness.} Suppose $M_D$ is $\cA$-provable with probability $1-\gamma$ over $x\sim \mu$, which is achieved by a function $A^* \in \cA$. If $M_D$ is also $\rho$-stable at $x$ w.r.t. $A^*$ with probability $1-\gamma$, then there exists $A \in \cA$ such that for all $B$ (of any size), with probability at least $1 - 2\gamma$
     \[
        \Ex{V^{A,B}(x)} \geq \frac{1}{\rho^{d+1}},
     \]
     where the expectation is over the internal randomness of the protocol only.
     \item \textbf{Soundness.} Let $\cB$ be the set of circuits of size $O\left(q 2^{3q}nd^2/\eta^2\right)$ that may query a function $A\in \cA$ at each gate. For every strategy $A\in \cA$, there exists $B\in\cB$ such that,
      \[
        \Ex[x\sim \mu]{V^{A,B}(x)} \leq \frac{\Pr[A(x) = L(x)]}{\rho^{d+1}} + \eta.
     \]
 \end{itemize}
\end{theorem}

Notice that we only require that $L$ admits an $\cA$-provable, $\rho$-stable decomposition in the completeness case. This means that a dishonest debater $A$ may produce an incorrect decomposition into queries $y^k$ for some input $x$ such that $M_D(x,y^k,z^k)$ is not $\rho$-stable, with answers to queries that cannot be answered by any function in $\cA$. Therefore, our protocol must account for situations where $A$ produces an unstable decomposition into hard-to-answer queries, and ensure $B$ is still able to win.

We usually consider $q$ being a constant independent of $n$ and $q^{d} = O(n^c)$. Since the dominant per-round cost is $B$'s probability estimates, a single prover-estimator debate has a cost of $O(q 2^{3q}nd^3/\eta^2)$ queries to functions in $\cA$. When there exists a very stable decomposition $M_D$ ($\rho \ll q$), prover-estimator debate has a polynomial improvement over the cost $O(n^c)$ required to verify the full human-judgeable argument given by the transcript of $M^{\horacle}(x,y^*)$. In particular, doubly-efficient debate would require $O(n^c)$ queries to the debaters, and naive recursive debate would either fail to satisfy soundness, or otherwise require expanding the whole $O(n^c)$-size search tree.

The proof of \pref{thm:cut-and-choose} follows from the completeness lemma \pref{lem:completeness} and the soundness lemma \pref{lem:soundness}. The proof of completeness, that an honest debater $A$ can always defeat a dishonest debater $B$, follows from the intuition that human judgement is used directly in the last round, and the reward in each round increases by a factor of $1/r$. Thus, even if $B$ lies throughout the first $k$ rounds, at some point $B$ (or eventually the human judgement oracle) must tell the truth, at which point a debater $A$ that has always answered honestly can pick up a positive reward that is larger than the sum of all negative rewards received so far.

The key technical piece is in showing soundness: that regardless of the dishonest strategy employed by debater $A$, debater $B$ can gain an advantage while using only slightly more compute than $A$. We use methods from online convex optimization to construct probabilities for $B$ that are indistinguishable by $A$ from the truth. The idea is to run online gradient descent on a sequence of loss functions which capture how well on average $A$ can distinguish $B$'s probabilities from the truth. The convergence rate of online gradient descent then yields bounds on the complexity of $B$'s strategy relative to $A$'s.

As a corollary of \pref{thm:cut-and-choose}, we show that answering correctly is an $A$-leading Stackelberg equilibrium in the prover-estimator debate game. The equilibrium result is weaker than \pref{thm:cut-and-choose}, because it requires $\cA$-provability always rather than only for completeness (a definitional artifact).

\begin{theorem}
\label{thm:stackelberg-eq}
Let $\cA$ be any set of functions on $\{0,1\}^n$ and let $d$, $q$, $\cB$ be as in \pref{thm:cut-and-choose}, with $\eta = 
\frac{\epsilon}{\rho^{d+1}}$. Let $\mu$ be a distribution on inputs $x\in\{0,1\}^n$, and let $L \in \clso(n^c)$ be a language with a recursive decomposition given by $M_D$.
Let $\cG$ be the two-player zero-sum game defined by the protocol of \pref{fig:protocol-cut-and-choose} with decomposition $M_D$, and let $\alpha < \frac{\eps}{\rho^{d+1}}$. If $M_D$ is $\cA$-provable and $\rho$-stable each with probability $1-\frac{\eps}{\rho^{d+1}}$, then in every $\alpha$-approximate, $A$-leading Stackelberg equilibrium of $\cG$
\[
\Pr_{x\sim \mu}[A(x) = L(x)] > 1 - 5\eps.
\]
\end{theorem}
We prove \pref{thm:stackelberg-eq} in \pref{sec:stackelberg-proof}. 

\subsection{Consequences for training with prover-estimator debate.}
The key take-away of our two main theorems is that one can train debaters in the prover-estimator debate protocol via standard methods that converge to Stackelberg equilibria in zero-sum games. If the validity of claims in the debate is not sensitive to a small set of evidence ($\rho$-stability) and the AI debaters only make claims they actually know the answer to ($\cA$-provability), then \pref{thm:stackelberg-eq} implies $A$ will produce the correct answer with probability at least $1 - 5\eps$.
On the other hand, if it is not possible for the AI's to decompose their arguments into subclaims which are both supported by sufficient evidence and which the AI's can actually answer, then the soundness case of \pref{thm:cut-and-choose} implies that the payoff to $A$ will fall below a known threshold. 
Thus, one can read off whether training has succeeded by looking at $A$'s expected payoff, and in the case of success, $A$ outputs the correct answer with probability $1 - 5\eps$.
We provide further recommendations for empirical debate research in \pref{sec:recommendations}.

\section{Technical Overview}
\label{sec:technical-overview}
In this section we summarize the main ingredients in the proof of \pref{thm:cut-and-choose}. 
We begin with a more detailed discussion of what goes wrong with naive recursive debate, and an idealized model of the obfuscated arguments problem.
\paragraph{Obfuscated arguments with naive recursion.}
In naive recursive debate, the debater $A$ decomposes the initial problem of deciding whether $x\in L$ into subproblems $y_1,\dots,y_q$ along with claimed answers $z_i = A(y_i)$ to each subproblem. Debater $B$ then selects a subproblem $y_i$, and the protocol recurses with $A$ claiming that $A(y_i) = z_i$.
The obfuscated arguments problem refers to the fact that $A$ may choose a decomposition $y_1,\dots,y_q$, and answers $z_i = A(y_i)$, such that only one of the answers is wrong, but it is computationally intractable for both $A$ and $B$ to determine which one.

A concrete example of obfuscation is primality testing~\citep{barnes2020obfuscated}. In this case, the initial question is whether an integer $x$ is prime or composite. The Miller-Rabin primality test can efficiently solve this problem, but there is no known polynomial-time algorithm to factor $x$ \citep{miller1975riemann,rabin1980probabilistic}. Indeed, if $x$ is the product of two randomly chosen $n$-bit primes, the RSA cryptosystem's security relies on the inability to efficiently factor $x$ \citep{rivest1978method}.
Hence, $A$ can employ the following obfuscating strategy in the case that $x$ is a hard-to-factor product of two primes. $A$ decomposes the integers from 1 to $x$ into $q$ intervals $y_1,\dots,y_q$, and claims that $x$ has no factors in each interval. When $B$ chooses an interval $y_i$ for recursion, $A$ further decomposes this into $q$ subintervals and so on. For $B$ to correctly identify the flaw in $A$'s argument, $B$ would need to find a factor of $x$, which is computationally intractable. This holds even though both $A$ and $B$ can run the Miller-Rabin test and be confident a factor exists. In summary, $A$ has efficiently decomposed the problem such that both debaters know there is a flaw somewhere in $A$'s argument, but neither can efficiently locate the flaw.

\paragraph{Soundness via indistinguishability.}
Prover-estimator debate aims to mitigate the obfuscated arguments problem using computational indistinguishability. Debater $A$ decomposes the problem $x$ into subproblems $y_1,\dots,y_q$, just as in naive recursive debate. Then debater $B$ is allowed to assign a joint probability distribution $\pq$ for the truth value of each subclaim $y_i$. Obfuscated arguments shows that $B$ cannot count on being able to efficiently solve the subproblems $y_i$. Instead, we show that $B$ can efficiently produce probabilities $\pq$ that are computationally indistinguishable from the correct answer by $A$. Our notion of efficiency is that $B$'s strategy can be computed by a small circuit making queries to $A$'s strategy at each gate.
The notion of indistinguishability that we use in the proof is known as \emph{outcome indistinguishability} \citep{dwork2021outcome}, and intuitively requires that, \emph{on average over the input distribution}, $A$ cannot tell the difference between a sample from the ground-truth answers and a sample from $B$'s chosen probabilities $\pq$. See \pref{sec:indistinguishability} for the formal definition.
To ensure that $A$ has advantage at most $\delta$ in distinguishing, the size of $B$'s circuit must grow as $O(q2^{3q}/\delta^2)$.
The exponential factor in $q$ appears because this type of indistinguishability requires explicitly representing $B$'s probability distribution over the $2^q$ possible answers for the $q$ subclaims at each level. Hence, we will require that $q$ is a constant independent of the input size $n$.

It is instructive to see how indistinguishability helps for a particularly simple type of obfuscated argument by $A$. Suppose that $A$ decomposes $x$ into $q$ subproblems so that there is one subproblem $y_i$ where the correct answer is $z_i=0$, but neither $A$ nor $B$ can identify this subproblem any better than by randomly guessing its location. In this case, $B$ can output a probability of $1/q$ that each answer $z_i = 0$. If $A$ chooses a subproblem $y_i$ to recurse on, $A$ will in all likelihood choose one where the true answer is 1, hence $B$'s probability is only off by $1/q$. If $A$ again obfuscates in the next round of the prover-estimator debate, then $B$'s probability will be off by at most $1/q^2$. Thus, if we reward $A$ proportional to the magnitude of the error in $B$'s probability, $A$ will receive a very small reward.

In the general case, prover-estimator debate gives rewards in every round of the debate. Prior to the initial round, $A$ receives expected reward proportional to how close $B$'s initial probability estimate is to $A$'s initial claim about the input. In all but the last round, the rewards are proportional to any error in the consistency between $B$'s probability claimed in the current round, versus the previous round. However, $A$ only receives these rewards if $A$ has correctly identified the direction of $B$'s error. In the final round, $B$'s probabilities are compared to human judgement, and again $A$ only receives a reward proportional to the magnitude of $B$'s error if $A$ correctly predicted the error direction. Hence, $B$ can use indistinguishability to ensure that $A$ gains rewards at most $\delta$ in each round, at the price of using strategies with circuit size $O(q2^{3q}/\delta^2)$. The only remaining source of reward for $A$ then comes from the initialization, where $A$ receives rewards depending on how closely $A$ agrees with $B$'s initial probability estimate $p$. Indistinguishability then implies that the fraction of inputs on which $A$ agrees with $B$'s initial estimate $p$ is approximately equal to the fraction of inputs where $A$ agrees with the ground truth. Thus, by using more compute $B$ can ensure that $A$'s rewards are bounded by a value proportional to the probability that $A$ agrees with the ground truth on the initial problem.

\paragraph{Completeness via stability.}
For the soundness guarantees upper-bounding $A$'s rewards to be meaningful, we must also show that if $A$ tells the truth, it is possible to lower-bound the rewards received by $A$. Completeness is achieved via two key properties. First, we will require $\rho$-stability of the decomposition, and second, the magnitude of available rewards will grow by a factor of $1/r$ in each round, where $r = 1/\rho$. To see why we need both properties, consider the following sketch of a completeness argument. Suppose $A$ always correctly answers the initial input problem and all subclaims, and that $A$ always selects the marginal probability from $B$ that is furthest from the truth. Then if $B$ ever outputs an incorrect probability, it will eventually lead to $B$ being caught: either in the final round when comparing to human judgement, or in some intermediate round where $B$ switches to telling the truth, and thus outputs an inconsistent probability.
The reason for the rewards to grow by a factor of $1/r$ in each round is to ensure that no matter what happened in the prior rounds, the cumulative negative rewards received earlier by $A$ are outweighed by the reward in the round where $B$ is finally caught in an inconsistency or by the human-judgement comparison.

The reason for the stability requirement arises from the need to check for inconsistencies between $B$'s probability for the query $\hat{y}$ selected for recursion in the previous round, when compared to $B$'s probabilities for the subclaims $y_i$ in the current round. In the worst case, very small changes in $B$'s probabilities for the subclaims could cause large changes in overall outcome for the query $\hat{y}$. Thus, even if $B$'s probability estimate for $\hat{y}$ was off by a lot, $B$ could produce estimates for the $y_i$ that are individually only slightly off (say $O(1/q)$), but are approximately consistent with the original estimate for $\hat{y}$. This would mean that, even when $A$ told the truth and $B$'s estimates of the initial probability were far from correct, $B$ could shift to approximately correct estimates while paying a small cost in rewards to $A$ for the consistency checks. Stability prevents this failure mode, because it requires that small errors in probability for the $y_i$ have a small impact on the outcome for $\hat{y}$.

Put another way, stability limits how rapidly errors can grow in the debate tree. This can also be interpreted as on average limiting the total search space necessary to determine correctness, because even errors of magnitude $\epsilon > 1/q$ cannot change the overall outcome by much, so it is not information-theoretically necessary to search the full debate tree of size $q^d$ to determine correctness.

\paragraph{Ensuring large reward gap.}

This outlined argument ensures $A$ receives some minimal reward when telling the truth, and that $B$ can spend more compute to ensure that $A$ receives an amount $\Delta$ less reward if $A$ frequently lies. However, $\Delta > 0$ is insufficient for efficient debate training. To see why, observe that if we run naive recursive debate with $q$ subclaims per round and $d$ rounds, we achieve $\Delta = 1/q^d$ if $B$ uniformly randomly selects subclaims for recursion. This does not enable efficient training, as in expectation one needs to run $q^d$ debates before seeing a single one where $B$ catches a flaw. In contrast, \pref{thm:cut-and-choose} obtains $\Delta = \Omega(1/\rho^d)$ and one need only see $O(\rho^d)$ debates before getting a non-trivial training signal, where $\rho$ can be significantly smaller than $q$ for stable decompositions.

\section{Discussion and Open Problems}
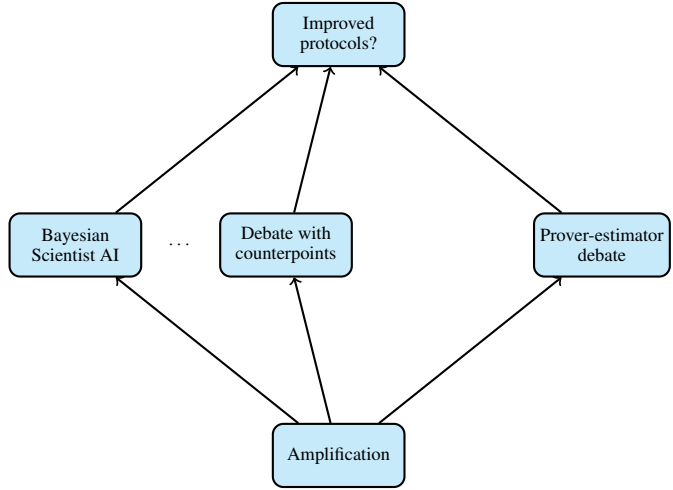
\begin{figure}[ht!]
\centering
\begin{tikzpicture}[
    scale=0.7, transform shape,
    protocol_node/.style={rectangle, rounded corners, draw, fill=cyan!20!white, minimum width=2.5cm, minimum height=1.2cm, align=center},
    thick
]

\node[protocol_node] (improved) at (0, 4) {Improved\\protocols?};

\node[protocol_node] (amplification) at (0, -4) {Amplification};

\node[protocol_node] (probabilistic) at (-5, 0) {Bayesian\\Scientist AI};

\node[protocol_node] (debate) at (-1, 0) {Debate with\\counterpoints};

\node[protocol_node] (prover) at (5, 0) {Prover-estimator\\debate};

\node[font=\normalsize] at (-3, 0) {$\cdots$};

\draw[->] (probabilistic) -- (improved);
\draw[->] (debate) -- (improved);
\draw[->] (prover) -- (improved);
\draw[->] (amplification) -- (probabilistic);
\draw[->] (amplification) -- (debate);
\draw[->] (amplification) -- (prover);

\end{tikzpicture}
\caption{While prover-estimator debate mitigates the obfuscated arguments problem that occurs in amplification and other scalable oversight protocols, it does not have all the features of these other protocols. For example, it lacks the argument/counterargument structure of \cite{irving2018ai} and the Bayesian structure of \cite{bengio2025scientist}. We hope that future work can explore whether the core ideas needed to avoid obfuscated arguments can be combined with features of other protocols.}
\label{fig:improved-protocols}
\end{figure}

Our paper presents debate protocols which circumvent the obfuscated arguments problem in recursive debate. Intuitively, this demonstrates that it is possible to design the rules of a debate such that training to win at debate leads to correct behavior of the final trained models in most cases.
These theoretical results lead to qualitative recommendations for empirical debate research, which present interesting avenues for future research in both of these directions.

\paragraph{Theoretical.} 
\Znote{This part needs to be revised}\Jnote{Done. What do you think?}
Our results identify stability as a critical requirement for debate training to succeed, and hence we need to better understand when stability can be achieved, either in theory for interesting classes of problems, or in practice for debates in messy, unformalized settings. On the theoretical side, \pref{thm:cut-and-choose} requires stability for completeness but not soundness, and thus is safe to use even if one is not confident in stability. 
However, it can still be the case that the prover only knows how to produce an unstable decomposition into solvable subproblems, implying that certain problems cannot be accurately judged via our protocol even when both debaters know the correct answer. Thus, an interesting direction for future work is to understand whether the stability condition is necessary in this setting. An optimistic hope might be that all relevant problems can be divided into two classes: those that require very precise dependence on individual subclaims (e.g. mathematical proofs) and those that do not (e.g. scientific knowledge obtained by noisy empirical evidence). If this were the case, future debate protocols might attempt to handle each type of claim separately. For example, one could require the debaters to produce machine-verifiable proofs in Lean for all  mathematical subclaims, and use prover-estimator debate to verify fuzzier subclaims. Theoretical progress in this direction depends on precisely characterizing the structure of problems that admit stable decompositions.

Another limitation is that our results hold in an average-case rather than a worst-case sense: we do not prove the judge is always able to tell which debater is correct on every question. Rather, we show that the payoffs given by the judge incentivize the debaters to converge to answering correctly. 
Thus, the  payoff assigned by the judge is a random variable that is correlated with correctness, but imperfectly so. 
We believe such an average-case analysis is necessary, and formalizing this is an interesting open question. We are also curious whether there are corresponding lower bounds to $B$'s circuit size if the stability parameter $\rho$ is held constant, in either the oracle or fine-grained complexity settings.

\paragraph{Empirical.}
The central empirical question raised by our paper is how to incorporate the qualitative recommendations from our theoretical results into practical debate training experiments with LLMs.
Such experiments may require additional flexibility for training to succeed, so we enumerate the key qualitative properties of prover-estimator debate that could guide practical debate experiments:
\begin{enumerate}
\item The prover $A$ both decomposes problems and chooses where to recurse given $B$'s estimates. The estimator $B$ only needs to express reasonable uncertainty about where a mistake might lie.
\item There is a per-round reward based on a consistency check between $B$’s estimate from the last round for the current main claim, and estimates for claims in the current round. $A$ can only collect this reward by accurately predicting the direction of $B$'s error between rounds.
\item In the last round, the judge directly decides if $B$’s estimates are correct.
\item The debaters should be trained to converge to an $A$-leading Stackelberg equilibrium.
\end{enumerate}

\paragraph{Zooming out.}
The broad question is whether some variant of debate works as a component of an overall alignment strategy. \cite{buhl2025alignment} sketches what an overall safety case using debate might look like, combining both theoretical and empirical evidence. We believe prover-estimator debate is a useful contribution on the theory side, but hope that future work finds improved protocols that combine the advantages of other scalable oversight methods (\pref{fig:improved-protocols}). In particular, we are interested in protocols which mitigate systematic errors in human input~\citep{irving2025systematic} and the potential for debaters to hide information in the subclaims $y_k$~\citep{pfau2025search}. These are just two areas for improvement: the full safety case sketch has a lot more holes to fill.
\section*{Acknowledgments}
We would like to thank Beth Barnes for good probing on the stability requirement, Marie Buhl for discussions relating this work and broader debate safety cases, Benjamin Hilton, Jacob Pfau, Zachary Kenton, Simon Marshall, Jacob Hilton, and Mario Szegedy for detailed comments on the paper, and Lean for refusing to buy the second author’s incomplete understandings of earlier versions of the proof.
JL is supported by Scott Aaronson's Coefficient Giving grant. 
ZX is supported by the NSF AI Institute for Foundations of Machine Learning (IFML).

\bibliographystyle{plainnat}
\bibliography{standard}

@article{irving2018ai,
      title={{AI} safety via debate}, 
      author={Geoffrey Irving and Paul Christiano and Dario Amodei},
      journal={arXiv preprint arXiv:1805.00899},
      year={2018},
      eprint={1805.00899},
      archivePrefix={arXiv},
      primaryClass={stat.ML}
}

@article{waldchen2022merlin,
  title={{Merlin-Arthur} Classifiers: Formal Interpretability with Interactive Black Boxes},
  author={W{\"a}ldchen, Stephan and Sharma, Kartikey and Zimmer, Max and Pokutta, Sebastian},
  journal={arXiv preprint arXiv:2206.00759},
  year={2022}
}

@article{leike2018scalable,
  title={Scalable agent alignment via reward modeling: a research direction},
  author={Leike, Jan and Krueger, David and Everitt, Tom and Martic, Miljan and Maini, Vishal and Legg, Shane},
  journal={arXiv preprint arXiv:1811.07871},
  year={2018}
}

@article{christiano2018supervising,
  title={Supervising strong learners by amplifying weak experts},
  author={Christiano, Paul and Shlegeris, Buck and Amodei, Dario},
  journal={arXiv preprint arXiv:1810.08575},
  year={2018}
}

@misc{barnes2020obfuscated,
  author = {Barnes, Beth},
  title = {Debate Update: Obfuscated Arguments Problem},
  year = 2020,
  url = {https://www.alignmentforum.org/posts/PJLABqQ962hZEqhdB/debate-update-obfuscated-arguments-problem},
  urldate = {2020-12-23}
}

@inproceedings{feige1997making,
  title={Making games short},
  author={Feige, Uriel and Kilian, Joe},
  booktitle={Proceedings of the twenty-ninth annual ACM symposium on Theory of computing},
  pages={506--516},
  year={1997}
}

@inproceedings{khan2024debating,
  title={Debating with More Persuasive {LLMs} Leads to More Truthful Answers},
  author={Khan, Akbir and Hughes, John and Valentine, Dan and Ruis, Laura and Sachan, Kshitij and Radhakrishnan, Ansh and Grefenstette, Edward and Bowman, Samuel R and Rockt{\"a}schel, Tim and Perez, Ethan},
  booktitle={Forty-first International Conference on Machine Learning},
  year={2024}
}

@article{michael2023debate,
  title={Debate helps supervise unreliable experts},
  author={Michael, Julian and Mahdi, Salsabila and Rein, David and Petty, Jackson and Dirani, Julien and Padmakumar, Vishakh and Bowman, Samuel R},
  journal={arXiv preprint arXiv:2311.08702},
  year={2023}
}

@inproceedings{zinkevich2003online,
  title={Online convex programming and generalized infinitesimal gradient ascent},
  author={Zinkevich, Martin},
  booktitle={Proceedings of the 20th International Conference on Machine Learning},
  pages={928--936},
  year={2003}
}

@article{hazan2016introduction,
  title={Introduction to online convex optimization},
  author={Hazan, Elad and others},
  journal={Foundations and Trends in Optimization},
  volume={2},
  number={3-4},
  pages={157--325},
  year={2016},
  publisher={Now Publishers, Inc.}
}

@InProceedings{brown-cohen24scalable,
  title = 	 {Scalable {AI} Safety via Doubly-Efficient Debate},
  author =       {Brown-Cohen, Jonah and Irving, Geoffrey and Piliouras, Georgios},
  year = 	 {2024},
  booktitle={Forty-first International Conference on Machine Learning},
}

@inproceedings{kentonscalable,
  title={On scalable oversight with weak {LLMs} judging strong {LLMs}},
  author={Kenton, Zachary and Siegel, Noah Yamamoto and Kramar, Janos and Brown-Cohen, Jonah and Albanie, Samuel and Bulian, Jannis and Agarwal, Rishabh and Lindner, David and Tang, Yunhao and Goodman, Noah and others},
  booktitle={Forty-first International Conference on Machine Learning},
  year={2024}
}

@article{canetti2013refereed,
  title={Refereed delegation of computation},
  author={Canetti, Ran and Riva, Ben and Rothblum, Guy N},
  journal={Information and Computation},
  volume={226},
  pages={16--36},
  year={2013},
  publisher={Elsevier}
}

@article{amit2024models,
  title={Models That Prove Their Own Correctness},
  author={Amit, Noga and Goldwasser, Shafi and Paradise, Orr and Rothblum, Guy},
  journal={arXiv preprint arXiv:2405.15722},
  year={2024}
}

@article{anil2021learning,
  title={Learning to give checkable answers with prover-verifier games},
  author={Anil, Cem and Zhang, Guodong and Wu, Yuhuai and Grosse, Roger},
  journal={arXiv preprint arXiv:2108.12099},
  year={2021}
}

@article{kirchner2024prover,
  title={Prover-verifier games improve legibility of {LLM} outputs},
  author={Kirchner, Jan Hendrik and Chen, Yining and Edwards, Harri and Leike, Jan and McAleese, Nat and Burda, Yuri},
  journal={arXiv preprint arXiv:2407.13692},
  year={2024}
}

@article{bengio2025scientist,
  title={Superintelligent Agents Pose Catastrophic Risks: Can Scientist {AI} Offer a Safer Path?},
  author={Bengio, Yoshua and Cohen, Michael and Fornasiere, Damiano and Ghosn, Joumana and Greiner, Pietro and MacDermott, Matt and Mindermann, S{\"o}ren and Oberman, Adam and Richardson, Jesse and Richardson, Oliver and others},
  journal={arXiv preprint arXiv:2502.15657},
  year={2025}
}

@inproceedings{Duchi2008EfficientPO,
  title={Efficient projections onto the l1-ball for learning in high dimensions},
  author={John C. Duchi and Shai Shalev-Shwartz and Yoram Singer and Tushar Chandra},
  booktitle={International Conference on Machine Learning},
  year={2008},
  url={https://api.semanticscholar.org/CorpusID:1226433}
}

@inproceedings{dwork2021outcome,
  title={Outcome indistinguishability},
  author={Dwork, Cynthia and Kim, Michael P and Reingold, Omer and Rothblum, Guy N and Yona, Gal},
  booktitle={Proceedings of the 53rd Annual ACM SIGACT Symposium on Theory of Computing},
  pages={1095--1108},
  year={2021}
}

@inproceedings{Dwork2022BeyondBG,
  title={Beyond {Bernoulli}: Generating Random Outcomes that cannot be Distinguished from Nature},
  author={Cynthia Dwork and Michael P. Kim and Omer Reingold and Guy N. Rothblum and G. Yona},
  booktitle={International Conference on Algorithmic Learning Theory},
  year={2022},
  url={https://api.semanticscholar.org/CorpusID:247682041}
}

@article{chandra1981alternation,
author = {Chandra, Ashok K. and Kozen, Dexter C. and Stockmeyer, Larry J.},
title = {Alternation},
year = {1981},
issue_date = {Jan. 1981},
publisher = {Association for Computing Machinery},
address = {New York, NY, USA},
volume = {28},
number = {1},
issn = {0004-5411},
url = {https://doi.org/10.1145/322234.322243},
doi = {10.1145/322234.322243},
journal = {J. ACM},
month = jan,
pages = {114–133},
numpages = {20}
}

@inproceedings{hammondneural,
  title={Neural Interactive Proofs},
  author={Hammond, Lewis and Adam-Day, Sam},
  booktitle={The Thirteenth International Conference on Learning Representations},
  year={2025}
}

@inproceedings{miller1975riemann,
  title={Riemann's hypothesis and tests for primality},
  author={Miller, Gary L},
  booktitle={Proceedings of the seventh annual ACM symposium on Theory of computing},
  pages={234--239},
  year={1975}
}

@article{rabin1980probabilistic,
  title={Probabilistic algorithm for testing primality},
  author={Rabin, Michael O},
  journal={Journal of number theory},
  volume={12},
  number={1},
  pages={128--138},
  year={1980},
  publisher={Elsevier}
}

@article{rivest1978method,
  title={A method for obtaining digital signatures and public-key cryptosystems},
  author={Rivest, Ronald L and Shamir, Adi and Adleman, Leonard},
  journal={Communications of the ACM},
  volume={21},
  number={2},
  pages={120--126},
  year={1978},
  publisher={ACM New York, NY, USA}
}

@article{buhl2025alignment,
  title={An alignment safety case sketch based on debate},
  author={Buhl, Marie Davidsen and Pfau, Jacob and Hilton, Benjamin and Irving, Geoffrey},
  journal={arXiv preprint arXiv:2505.03989},
  year={2025}
}

@misc{irving2025systematic,
  author = {Irving, Geoffrey},
  title = {Dodging systematic human errors in scalable oversight},
  year = 2025,
  url = {https://www.alignmentforum.org/posts/EgRJtwQurNzz8CEfJ/dodging-systematic-human-errors-in-scalable-oversight},
  urldate = {2025-5-14}
}

@misc{pfau2025search,
  author = {Pfau, Jacob and Irving, Geoffrey},
  title = {Unexploitable search: blocking malicious use of free parameters},
  year = 2025,
  url = {https://www.alignmentforum.org/posts/CuneN5HmLnztsLRzD/unexploitable-search-blocking-malicious-use-of-free-1},
  urldate = {2025-5-21}
}

@article{chen2023playing,
  title={Playing Large Games with Oracles and {AI} Debate},
  author={Chen, Xinyi and Chen, Angelica and Foster, Dean and Hazan, Elad},
  journal={arXiv preprint arXiv:2312.04792},
  year={2023}
}

@article{KulkarniTal2016fbs,
  author  = {Raghav Kulkarni and Avishay Tal},
  title   = {On Fractional Block Sensitivity},
  journal = {Chicago Journal of Theoretical Computer Science},
  volume  = {2016},
  number  = {8},
  pages   = {1--16},
  year    = {2016},
  doi     = {10.4086/cjtcs.2016.008}
}
\newpage
\appendix
\onecolumn
\section{Proof of \pref{thm:cut-and-choose}}
\label{sec:proof}
Fix a strategy $A$ for the leader. Our approach is to construct a small circuit implementing $B$'s strategy via online gradient descent on a sequence of loss functions chosen so that $B$'s probabilities will be indistinguishable by $A$ from the truth.
In each iteration of online gradient descent the loss is chosen by measuring, averaged over all queries, how well $A$ can distinguish between $B$'s current strategy and the true answer. 

Crucially, \emph{our online gradient descent algorithm cannot be run}: each step of gradient descent requires access to the truth, which is available only with exponential compute. However, we will show that the resulting circuit $B$ has complexity bounded in terms of \emph{number of iterations}, which will be small. 

\subsection{Indistinguishability}
\label{sec:indistinguishability}
We begin with a formal definition of the notion of indistinguishability that we will use in our proof.
\begin{definition}
\label{def:indistinguishability}
    Let $X$ be a set, $Z$ be a finite set, $\mu$ be a finitely supported probability distribution on $X$, and $g:X \to \Delta_Z$.
    Let $\func$ be a finite family of functions $f:X\times Z\times \Delta_Z \to [-1,1]$.
    A function $h:X\to \Delta_Z$ is \emph{$(\delta,\func)$-indistinguishable} from $g$ if and only if for all $f\in\func$
    \[
        \Abs{\Ex[\substack{x \sim \mu\\ z\sim g(x)}]{f(x,z,h(x))} - \Ex[\substack{x \sim \mu\\z\sim h(x)}]{f(x,z,h(x))}} < \delta
    \]
\end{definition}
Intuitively, \pref{def:indistinguishability} says that test-functions from the class $\func$ cannot tell the difference between $h$ and $g$.
Our notion of indistinguishability is nearly equivalent to a multi-class variant of \emph{outcome indistinguishability} \citep{dwork2021outcome}, with the only difference being that we use real-valued bounded test functions, rather than functions taking values in $\{0,1\}$. The work of \cite{Dwork2022BeyondBG} further provides an algorithm for constructing multi-class outcome indistinguishability in the $\{0,1\}$-valued test function setting. However, due to a slightly different focus, \cite{Dwork2022BeyondBG} achieve a linear dependence on $\Abs{Z}$ in the complexity of $h$, at the cost of an $O(\log\Abs{\cA})$ factor. In our setting $\Abs{Z} = 2^q$ is small, while $\log\Abs{\cA}$ is as large as the circuit complexity (or number of weights in the neural network) of the debaters. Thus, we design a different algorithm than that of \cite{Dwork2022BeyondBG} which has quadratic dependence on $\Abs{Z}$, but is independent of $\Abs{\cA}$.

Our main technical lemma uses online gradient descent to produce a $(\delta,\func)$-indistinguishable function $h$ that has bounded complexity relative to $\func$. In particular, $h$ will be computable by a small circuit that at each gate may query a function $f\in\func$.
The result will follow from online gradient descent, where in each iteration we select a loss function determined by some $f\in\func$ such that the condition of \pref{def:indistinguishability} does not hold.

We recall the standard guarantees of online gradient descent \cite{zinkevich2003online, hazan2016introduction}.
\begin{theorem}[Online Gradient Descent]
\label{thm:online-gradient-descent}
Let $K$ be a convex set, $D$ be the diameter of $K$, and $G$ an upper bound on $\norm{\nabla l_t(v)}_2$ for $v\in K$. Then for each $v\in K$, online gradient descent with step sizes $\eta_t = \frac{D}{G\sqrt{t}}$ satisfies
\[
    \textrm{regret} = \frac{1}{T}\sum_{t=1}^T l_t(v_t) - \frac{1}{T}\sum_{t=1}^Tl_t(v) \leq \frac{3DG}{2\sqrt{T}}
\]
where the loss $l_t$ can be chosen adversarially after $v_t$ is revealed.
\end{theorem}

\begin{algorithm}[htbp]
\caption{Constructing an indistinguishable $h$}\label{alg:indistinguishability}
\begin{algorithmic}[1]
\State{\textbf{Input:} Finite sets $X$ and $Z$, a distribution $\mu$ on $X$, a function $g:X\to \Delta_Z$, a class $\func$ of functions $f:X\times Z \times \Delta_Z\to [-1,1]$, and accuracy parameter $\delta > 0$.} 
\State{\textbf{Output:} A function $h:X \to \Delta_Z$ that is $(\delta,\mathcal F)$-indistinguishable from $g$.}
\State{Online gradient descent with each loss given by a function $f\in \func$ that can distinguish $h$ from $g$.}
\State{Initialize $h_1:X\to \Delta_Z$ to $h_1 \equiv \frac{1}{\Abs{Z}}$} \Comment{Initialize estimates to the uniform distribution.}
\For{$t = 1$ to $T$}
    \State{For any $f\in \func$ let \[\Adv^t_f = \Ex[\substack{x \sim \mu\\ z\sim g(x)}]{f(x,z,h_t(x))} - \Ex[\substack{x \sim \mu\\z\sim h_t(x)}]{f(x,z,h_t(x))}\]}
    \If{there exists $f\in\func$ such that $\Abs{\Adv^t_{f}} \geq \delta$} \Comment{If $f$ can distinguish $h$ from $g$.}
        \State{$\displaystyle h_{t+1}(x,z) = h_t(x,z) + \eta_t\cdot \sign \Adv^t_{f}\cdot f(x,z,h_t(x))~\forall ~x,z$. \Comment{Gradient step with $f$.}}
        \State{Project all outputs of $h_{t+1}$ to $\Delta_Z$. \Comment{Project to probabilities.}}
    \Else{}
        \State{\Return{$h_t$}}
    \EndIf
\EndFor
\end{algorithmic}
\end{algorithm}

\begin{lemma}[Indistinguishability]
\label{lem:indistinguishability}
Let $X$ be a set, $Z$ be a finite set, $\mu$ be a finitely supported distribution on $X$, and $\func$ a set of functions $f : X \times Z \times \Delta_Z \to 
[-1,1]$. For any $g:X \to \Delta_Z$ there exists a function $h:X\to \Delta_Z$ such that
\begin{enumerate}
    \item $h$ is $(\delta,\func)$-indistinguishable from $g$.
    \item $h$ is computable by a circuit of size $O(\Abs{Z}^2/\delta^2)$ that can query some $f\in\func$ at each gate.
\end{enumerate}
\end{lemma}
\begin{proof}
    The function $h$ will be the output of \pref{alg:indistinguishability}.
    The first condition of the lemma will follow from the guarantees of online gradient descent, so we begin by showing that \pref{alg:indistinguishability} is online gradient descent on an appropriate sequence of losses.

    For a function $h:X\to \Delta_Z$ we let $h(x,z) = \Pr[h(x) = z]$ denote the probability assigned to $z\in Z$ by $h$ at input $x \in X$. We can interpret both $h(x,z)$ and $F(x,z) = f(x,z,h(x))$ as functions $X\times Z \to \R$. 
    Construct an embedding $\Psi$ sending functions $h:X\times Z\to \R$ to vectors in $\R^{X\times Z}$ as
    \[
        \Psi h(x,z) = \sqrt{\mu(x)} \cdot h(x,z)
    \]
    Hence the inner product between embeddings is given by $\iprod{\Psi h,\Psi g} = \Ex[x\sim \mu]{\sum_{z\in Z}h(x,z)g(x,z)}$.
    Let $\cH = \{\vec{h} \in \R^{X\times Z}\ \mid \vec{h} = \Psi h, h:X \to \Delta_Z\}$.
    Observe that $\norm{\vec{h}}_2 \leq 1$ for all $\vec{h} \in \cH$ and that the $\ell_2$-projection onto $\cH$ is given by, for each $x$, projecting the $\Abs{Z}$-dimensional vector $\vec{h}(x,\cdot)$ to $\Delta_Z$.

    Now for any $f\in\func$ we define the loss functions $l_t:\R^{X\times Z} \to \R$ by
    \begin{equation}
        \label{eqn:indist-loss}
        l_t(\vec{h}) = \Abs{\iprod{\Psi g,\Psi F} - \iprod{\vec{h},\Psi F}}
    \end{equation}
    As it is the absolute value of an affine function, $l_t$ is convex in $\vec{h}$. Furthermore, for any $\vec{h} \in \cH$ with $\vec{h} = \Psi h$, the subgradient of $l_t$ is given by
    \begin{align*}
        \partial l_t(\vec{h})(x,z) 
            &= -\sign\left(\iprod{\Psi g, \Psi F} - \iprod{ \vec{h}, \Psi F}\right)\Psi F(x,z)\\
            &= -\sign\left(\Ex[\substack{x \sim \mu\\z \sim g(x)}]{f(x,z,h(x))} - \Ex[\substack{x \sim \mu\\ z \sim h(x)}]{f(x,z,h(x))}\right)\Psi F(x,z)\\ 
            &= -\sign \Adv_f \cdot \Psi F(x,z)
    \end{align*}
    Hence the embedding by $\Psi$ of the gradient update on line 8 of \pref{alg:indistinguishability} can be written as 
    \[
        \Psi h_{t+1} = \Psi h_t + \eta_t \cdot \sign \Adv^t_f \cdot \Psi F = \Psi h_t - \eta_t \cdot \partial{l_t(\Psi h_t})
    \]
    Furthermore, the projection on line 9   corresponds to the projection of $\vec{h_t}$ onto $\cH$. Therefore, \pref{alg:indistinguishability} is precisely online gradient descent on the embedded vectors $\Psi h_t$ with loss $l_t$. Since $\Abs{\Adv_f} \leq 2$ we have $\norm{\partial l_t(\vec h)}_2 \leq 2\norm{\Psi F}_2 \leq 2\sqrt{\Abs{Z}}$ for all $\vec{h}\in\cH$.
    By \pref{thm:online-gradient-descent} we conclude
    \begin{align}
    \label{eqn:ogd-vec-guarantee}
            \frac{3\sqrt{\Abs{Z}}}{\sqrt{T}} &\geq \frac{1}{T}\sum_{t=1}^T l_t(\vec{h}_t) - \frac{1}{T}\sum_{t=1}^Tl_t(\Psi g)\nonumber\\
                &= \frac{1}{T}\sum_{t=1}^T \Abs{\iprod{\Psi g,\Psi F} - \iprod{\vec{h},\Psi F}}\nonumber\\
                &= \frac{1}{T}\sum_{t=1}^T \Abs{\Ex[\substack{x \sim \mu\\ z\sim g(x)}]{f_t(x,z,h_t(x))} - \Ex[\substack{x \sim \mu\\z\sim h_t(x)}]{f_t(x,z,h_t(x))}}\nonumber\\
                &= \frac{1}{T}\sum_{t=1}^T \Abs{\Adv_f^t}
    \end{align}

    Each term in the final sum of \pref{eqn:ogd-vec-guarantee} is greater than or equal to $\delta$ by the if condition in \pref{alg:indistinguishability}.
    Hence \pref{alg:indistinguishability} can run for at most $T = O(\Abs{Z}/\delta^2)$ steps before it is no longer possible to find an $f$ satisfying $\Abs{\Adv_f^t} \geq \delta$, as otherwise this would contradict \pref{eqn:ogd-vec-guarantee}. Thus, the function $h$ output by \pref{alg:indistinguishability} is $(\delta,\func)$-indistinguishable from $g$, as desired.

    For the second condition of the lemma, we can construct a size $O(\Abs{Z}^2/\delta^2)$ circuit for $h$ as follows. On input $x\in X$, the function $h_{t+1}(x)$ can be computed by evaluating $h_t(x,z)+\eta_t\cdot \operatorname{sign}\operatorname{Adv}^t_f f(x,z,h_t(x))$
 for each $z\in Z$, and then performing the projection onto $\Delta_Z$. This takes $O(\Abs{Z})$ evaluations of $f$, one evaluation of $h_t$, and $O(\Abs{Z})$ time for the projection \citep{Duchi2008EfficientPO}. Thus, a circuit to compute the final $h$ returned by \pref{alg:indistinguishability} can be constructed starting from $h_1$ with size at most $O(T\cdot \Abs{Z}) = O(\Abs{Z}^2/\delta^2)$.   
\end{proof}

\subsection{Soundness}

Now that we have established that small circuits can produce indistinguishable probabilities for any class of test functions $\cF$, we will leverage \pref{lem:indistinguishability} to produce the probabilities for $B$ in the Prover-Estimator debate protocol.
The main idea is to, given a fixed strategy for $A$, construct $B$'s strategy in each round $k$ via \pref{lem:indistinguishability} with the function class $\cF$ chosen to be size $q$ circuits that make $q$ queries to $A$. The distribution on inputs for round $k$ is given by the fixed strategies for $A$ and $B$ in prior rounds applied to the original input distribution $\mu$.

\begin{lemma}[Soundness]
\label{lem:soundness}
For any $\eta > 0$ and any strategy $A \in \cA$ there exists a strategy $B$ computable by circuits of size $O(q 2^{3q}nd^2/\eta^2)$ that may query $A$ at each gate such that
\[
\Ex[x \sim \mu]{V^{A,B}(x)} < \frac{\Pr[A(x) = L(x)]}{\rho^{d+1}} + \eta.
\]
\end{lemma}
\begin{proof}
    
    We first show that at initialization, we can construct $B$'s strategy so that the probability $\hat p_d$ is indistinguishable by $A$ from the truth, and prove an upper bound on the initial reward $A$ can receive in terms of $\Pr[A(x) = L(x)]$.
    
    \textbf{Initialization.} Let $X = \{0,1\}^n$ be the set of possible inputs $x$, $Z = \{0,1\}$ be the set of outputs, and $\func$ be the set of functions $f:X \times Z \times \Delta_Z \to [-1,1]$ computable by $O(1)$-sized circuits that evaluate $A$ once. Let $g:X \to \Delta_Z$ be the function that puts probability 1 on $L(x)$ given input $x$. By \pref{lem:indistinguishability} there exists a function $h: X \to \Delta_{Z}$ that is $(\delta,\func)$-indistinguishable from $g$, and can be computed by $O(1/\delta^2)$-sized circuits that can query functions $A\in\cA$ in each step.
    Let the strategy for $B$ be given by $B(x) = \Pr_{z \sim h(x)}[z = 1]$.
    Note that the initial reward $\rinit(x,{z}) = {r^{d+1}\cdot \1(z = A(x))}$ can be computed by evaluating $A$ once, and performing a constant number of arithmetic operations, hence $\rinit \in \func$. Therefore by $(\delta,\func)$-indistinguishability of $h$ we have
    \begin{align*}
        \delta &> \Abs{\Ex[\substack{x \sim\mu\\ {z} \sim h(x)}]{\rinit(x,{z})} - \Ex[\substack{x \sim \mu\\ {z} \sim g(x)}]{\rinit(x,{z})}}\\
            & = \Abs{\Ex[\substack{x \sim\mu\\ {z} \sim h(x)}]{\rinit(x,{z})} - \Ex[\substack{x \sim \mu}]{r^{d+1}\cdot\1(A(x) = L(x))}}\\
            & = \Abs{\Ex[\substack{x \sim\mu}]{\rinit} - r^{d+1}\Pr_{\substack{x \sim \mu}}[A(x) = L(x)]}.
    \end{align*}
    Rearranging yields
    \begin{equation}
        \label{eqn:init-reward}
        \Ex[\substack{x \sim\mu}]{\rinit} \leq r^{d+1}\Pr_{\substack{x \sim \mu}}[A(x) = L(x)] + \delta.
    \end{equation}
    We next recursively construct $B$'s strategy as follows. Given fixed strategies for $A$ and $B$ in round $k' > k$, let $\mu_k$ be the distribution on inputs $x$ and queries $y$ in round $k$. 
    
    \textbf{Recursion.} Let $X = \{0,1\}^{n+n+nq}\times [0,1]$ be the set of input-query sequences $(x,\hat{y}^k,y^k)$ along with the claimed probability $\hat p_k$ from the previous round. Let $Z = \{0,1\}^q$ be the set of  query answers $z$. Let $\func$ be the set of functions $f:X\times Z \times \Delta_Z \to [-1,1]$ computable by size $O(nq)$ circuits that make $q$ queries to $A^k$, and define $g_k: X \to \Delta_Z$ by for every $\vx^k = (x,\hat{y}^k,y^k, \hat p_k) \in X$,
    \[
        \Pr_{z\sim g_k(\vx^k)}[z]  = \prod_{i=1}^q \1(g(x,y_i) = z_i),
    \]where $g$ is the function that assigns the correct answer to each query $y_i$.
        \pref{lem:indistinguishability} implies that there exists a function $h_k:X \to \Delta_Z$ that is $(\delta,\func)$-indistinguishable from $g_k$, and is computable by circuits of size $O(q2^{2q}n/\delta^2)$. We let $B$'s strategy in round $k$ be outputting the entire distribution $h_k(\vx^k)$, which can be implemented by a circuit of size $O(q2^{3q}n/\delta^2)$.

    In every round $k$, the reward $r_k = r^{k}A^k(x,\hat{y}^k, \hat p_k)(M_D(x,y^k,z^k) - \hat p_k)$ can be viewed as a function $r_k(\vx^k, z^k, h_k(\vx^k))$ that depends only on inputs $\vx^k = (x,\hat{y}^k,y^k, \hat p_k) \in X$, the probability distribution $h_k(x) \in\Delta_Z$ produced by $B_i^k$, and query answers $z \in Z$.
    Since $r_k$ can be computed by running the machine $M_D$ for $O(nq)$ steps, $r_k \in\cF$.
    Hence by $(\delta,\cF)$-indistinguishability of $h_k$,
    \begin{align}
        \label{eqn:reward-indisting-bound}
        \Abs{\Ex[\substack{\vx \sim \mu_k\\ z\sim h_k(x)}]{r_k(\vx,z,h_k(\vx))} - \Ex[\substack{\vx \sim \mu_k\\ z \sim g_k(\vx)}]{r_k(\vx,z,h_k(\vx))}} < \delta
    \end{align}
  
    Recalling that $\hat{y}^k = y^{k+1}_j$ and $z_j^{k+1}$ are the query and answer that $A$ selects for recursion in the previous round $k+1$, and that $\hat p_k = h_{k+1}(\vx^{k+1})$ is $B$'s previous round probability, the quantity
    \begin{equation}
        f_{k+1}(\vx^{k+1},z^{k+1},h_{k+1}(\vx^{k+1})) = r^{k}A^{k}(x,\hat{y}^{k},\hat p_{k})z^{k+1}_j
    \end{equation}
    can be viewed as function taking inputs in $X \times Z \times \Delta_Z$, and making two queries to $A$ in order to determine the index $j$ on which to recurse as well as the value $A^{k}(x,\hat{y}^{k},\hat p_{k})$. For the base case $k = d$, we abuse notation and let $z^{k+1}_j = \hat{z}$ and $h_{k+1}$ be the function $h$ used to construct the strategy $B(x)$ at initialization. 
    Hence $f_{k+1} \in \cF$.
    By $(\delta,\cF)$-indistinguishability of $h_{k+1}$ on the distribution $\mu_{k+1}$, we have that
    \begin{equation}
        \label{eqn:recursion-fk-bound}
        \Abs{\Ex[\substack{\vx \sim\mu_{k+1}\\ z \sim h_{k+1}(\vx)}]{f_{k+1}(\vx,z,h_{k+1}(\vx))} - \Ex[\substack{\vx \sim \mu_{k+1}\\ z \sim g_{k+1}(\vx)}]{f_{k+1}(\vx,z,h_{k+1}(\vx))}} < \delta
    \end{equation}
    Next, the definition of $r_k$ and the true answer distribution function $g_k$ imply that for $k \geq 1$
    \begin{align}
        \Ex[\substack{\vx \sim \mu_k\\ z \sim g_k(\vx)}]{r_k(\vx,z,h_k(\vx))} &= \Ex[\substack{\vx \sim \mu_k\\ z \sim g_k(\vx)}]{r^{k}A^k(x,\hat{y}^k, \hat p_k)(M_D(x,y^k,z^k) - \hat p_k)}\nonumber\\
        &= \Ex[\substack{\vx \sim \mu_k}]{r^{k}A^k(x,\hat{y}^k, \hat p_k)(g(x,\hat{y}^k) - \hat  p_k)}\nonumber\\
        &= \Ex[\substack{\vx \sim \mu_{k+1}\\ z \sim g_{k+1}(\vx)}]{f_{k+1}(\vx,z,h_{k+1}(\vx))} - \Ex[\substack{\vx \sim \mu_{k+1}\\ z \sim h_{k+1}(\vx)}]{f_{k+1}(\vx,z,h_{k+1}(\vx))}\nonumber\\
        &< \delta     \label{eqn:true-reward-bound}
    \end{align}
    where the final inequality follows from \pref{eqn:recursion-fk-bound}.
    Finally, in the base case $k=0$, we force $A$ to produce the subset $\tilde{u}_{S_i}$ of the implicit proof without additional context, and so $A$'s answers at the bottom level must correspond to a single consistent proof $\tilde{u}$. Hence, we have
    \begin{align}
    \label{eqn:true-reward-bound-base-case}
        \Ex[\substack{\vx \sim \mu_0\\ z \sim g_0(\vx)}]{r_0(\vx,z,h_0(\vx))} &= \Ex[\substack{\vx \sim \mu_0}]{A^0(x,\hat{y}^0,\hat p_0)(\horacle(\tilde{u}_{S_i}) - \hat p_0)}\nonumber\\
        &\leq \Ex[\substack{\vx \sim \mu_0}]{A^0(x,\hat{y}^0,\hat p_0)(g(x,\hat{y}^0) - \hat p_0)}\nonumber\\
        &= \Ex[\substack{\vx \sim \mu_{1}\\ z \sim g_{1}(\vx)}]{f_{1}(\vx,z,h_{1}(\vx))} - \Ex[\substack{\vx \sim \mu_{1}\\ z \sim h_{1}(\vx)}]{f_{1}(\vx,z,h_{1}(\vx))}\nonumber\\
        &< \delta.
    \end{align}
    Combining \pref{eqn:true-reward-bound} and \pref{eqn:true-reward-bound-base-case} with \pref{eqn:reward-indisting-bound} implies that for all $k\geq 0$
    \begin{equation}
        \Ex[\substack{\vx \sim \mu_k\\ z\sim h_k(x)}]{r_k(\vx,z,h_k(\vx))} < 2\delta.
    \end{equation}
    Summing over the $d$ rounds of the debate and adding the expected $\rinit$ from \pref{eqn:init-reward} gives 
    \[
    \Ex[x \sim \mu]{V^{A,B}(x)} < r^{d+1} \cdot {\Pr[A(x) = L(x)]} + (2d+1)\delta.
    \]
    Substituting $r = 1 / \rho$ and setting $\delta = {\eta}/(2d+1)$   
    completes the proof.
\end{proof}

\subsection{Completeness}\label{sec:proof:completeness}
In this section we prove the completeness case of \pref{thm:cut-and-choose}. The main idea is that if $x\in L$, for an $\cA$-provable, sufficiently stable language $L$, the strategy of $A$ correctly outputting proofs $y$ and pointing out inaccurate estimations by $B$ will always yield some positive reward for $A$.

\begin{lemma}[Completeness]
    \label{lem:completeness}
    Let $0 < \gamma < \frac{1}{2}$.
    Suppose $M_D$ is $\cA$-provable with probability $1-\gamma$ over $x\sim \mu$, which is achieved by a function $A^* \in \cA$. If $M_D$ is also $\rho$-stable at $x$ w.r.t. $A^*$ with probability $1-\gamma$, then there exists $A \in \cA$ such that for all $B$ (of any size), with probability at least $1 - 2\gamma$
    \[
        \Ex{V^{A,B}(x)} \geq \frac{1}{\rho^{d+1}}.
    \]
    where $V$ is the verifier for the protocol of \pref{fig:protocol-cut-and-choose} and the expectation is over the internal randomness of the protocol only.
\end{lemma}

\begin{proof}
    We take $A$ to be same as $A^*$ when decomposing the current problem into sub-problems. Therefore, $M_D$ is also $\rho$-stable at $x$ w.r.t. $A$ with probability $1-\gamma$ over $x \sim \mu$. By a union bound, with probability $1-2\gamma$ over $x \sim \mu$, 
    \begin{itemize}
        \item $M_D$ is $\rho$-stable at $x$ w.r.t. $A$;
        \item every sub-problem under the debate tree of $x$ generated by $A$ can be solved by $A$.
    \end{itemize}

    We now only consider $x$ that satisfies the conditions above. $A$'s strategy is rather simple:  
    In each round $k\leq d$, $A$ always outputs $1$ or $-1$ (for underestimate and overestimate) if the ground truth is $1$ or $0$ respectively. When $B$ outputs its estimation $p_k$, $A$ calculates the marginal probability for each sub-problem, and then chooses the sub-problem where $B$'s estimation (i.e., the marginal probability) is most inaccurate. 

    Note that if $B$ is honest, the reward $A$ could get is exactly $r^{d+1} = 1/{\rho^{d+1}}$. We now consider $B$ being dishonest strategically and 
    calculate the expected reward $B$ could gain compared to the honest baseline.
    Let $\varepsilon_{k}$ be the largest estimation error $B$ makes for sub-problems in level $k$. In particular, let $\varepsilon_{d+1}$ be the estimation error $B$ makes at the initialization.  

    \vspace{4pt}
    \textbf{Initialization level $k = d+1\,$:} If $B$'s estimation is $\varepsilon_{d+1}$ off from the ground truth, he can gain $r^{d+1} \cdot \varepsilon_{d+1}$ in expected reward.

    \vspace{4pt}
    \textbf{Intermediate level $k \in [d]\,$:} By symmetry, let us assume the correct value for the current problem $\hat{y}^{k}$ is $0$, and Alice outputs $-1$ claiming that $B$ overestimates the answer. By definition, the expected reward $B$ get for this level is 
    $r^{k} \cdot \left(M_D(x, y^{k}, z^{k}) - \varepsilon_{k+1}\right).$
    Since $M_D$ is $\rho$-stable at $(x, y^{k})$, the expected value of $M_D(x, y^{k}, z^{k})$ is at most $\rho \cdot \varepsilon_{k}$. Therefore, the expected reward $B$ could get at the intermediate level $k$ is $r^{k} \cdot \left(\rho \cdot \varepsilon_{k} - \varepsilon_{k+1}\right)$.

    \vspace{4pt}
    \textbf{Bottom level $k = 0\,$:} If $B$'s estimation is off by $\varepsilon_0$ at the bottom level, he will be caught immediately and lose $\varepsilon_0$ of reward.

    Now, take the sum of the expected reward $B$ could get for each level, we have
    
    \[ \varepsilon_0 + \sum_{k=1}^d r^{k} \cdot \left(\rho \cdot \varepsilon_{k} - \varepsilon_{k+1}\right) + r^{d+1}\cdot\varepsilon_{d+1} = (r^{d+1} - r^{d}) \cdot \varepsilon_{d+1} ,\]
    which is a negative value if $\varepsilon_{d+1} > 0$. Therefore, $B$'s best strategy against $A$ is being honest, and in this case $A$ gets $1/{\rho^{d+1}}$ reward.
\end{proof}

\subsection{Stability and fractional block sensitivity}
\label{sec:stability-fbs}
Interestingly, the $\rho$-stability condition of \pref{def:epsilon-stable} is exactly captured by the notion of \emph{fractional block sensitivity} from boolean function analysis. We first present the formal definition of fractional block sensitivity, and then explain why they are equivalent.

\begin{definition}[Fractional block sensitivity]
Let $f : \{0, 1\}^n \to \{0, 1\}$ and $x \in \{0, 1\}^n$. For a nonempty set $B \subseteq [n]$, let $x^B$ denote the string obtained from $x$ by flipping the bits in the coordinates of $B$, i.e.,
\[ (x^B)_i = \begin{cases} 1 - x_i & \text{if } i \in B, \\ x_i & \text{if } i \notin B. \end{cases} \]
A set $B \subseteq [n]$ is a sensitive block for $f$ at $x$ if $f(x) \neq f(x^B)$. Write $\mathcal{S}_x(f)$ for the family of all sensitive blocks at $x$.

The fractional block sensitivity of $f$ at $x$ is the optimal value of the linear program
\[ \mathrm{fbs}_x(f) = \max \left\{ \sum_{B \in \mathcal{S}_x(f)} w_B \;\middle|\; \begin{aligned} w_B &\ge 0 \text{ for all } B \in \mathcal{S}_x(f), \\ \sum_{\substack{B \in \mathcal{S}_x(f) \\ i \in B}} w_B &\le 1 \quad \text{for all } i \in [n] \end{aligned} \right\}. \]
The fractional block sensitivity of $f$ is
\[ \mathrm{fbs}(f) = \max_{x \in \{0, 1\}^n} \mathrm{fbs}_x(f). \]
\end{definition}

\begin{observation}
$M_D$ is $\rho$-stable at $(x, y)$ if and only if $\mathrm{fbs}_z(M_D(x, y, \cdot)) \le \rho$, where $z$ is the ground truth solution for $(x, y)$.
\end{observation}

\begin{proof}
Fix any $(x, y)$ and $\epsilon$ and assume that the ground truth solution for $(x, y)$ is $\mathbf{0}$ without loss of generality. From Bob's perspective, choosing a distribution $q$ over $\{0, 1\}^q$ is picking an assignment to $2^q$ variables $\{w_B\}_{B \in \{0, 1\}^q}$, where
\[ w_B := \Pr_{z \sim q} [z = B]. \]
For each position $i \in [q]$, the marginal probability constraint says that
\[ \sum_{i \in B} w_B \le \epsilon. \]
The probability that Bob flips the correct answer is exactly
\[ \sum_{B \in \mathcal{S}_x(f)} w_B. \]
It is easy to see that this is equivalent to the linear optimization problem for $\mathrm{fbs}_{\mathbf{0}}(M_D(x, y, \cdot))$.
\end{proof}

\section{Proof of \pref{thm:stackelberg-eq}}
\label{sec:stackelberg-proof}
The proof of \pref{thm:stackelberg-eq} follows from an analysis of Stackelberg equilibria via the soundness and completeness of \pref{lem:soundness} and
\pref{lem:completeness}.
\begin{proof}[Proof of \pref{thm:stackelberg-eq}]
    Let $A \in \cA$ and $B \in \cB$ be the respective strategies of debaters $A$ and $B$ in any $\alpha$-approximate, $A$-leading Stackelberg equilibrium in the protocol of \pref{fig:protocol-cut-and-choose}. 
    Suppose that $A$ is dishonest with probability $\gamma$ i.e.
    \begin{equation*}
        \Pr_{x\sim \mu}\left[A(x) \neq L(x) \right] = \gamma > 0
    \end{equation*}
    Then by \pref{lem:soundness} there exists a strategy $B'\in \cB$ , such that 
    \begin{equation*}
        \Ex[x \sim \mu]{ V^{A,B',\horacle}(x) } \leq \frac{1-\gamma}{\rho^{d+1}} + \eta.
    \end{equation*}
    The definition of $\alpha$-approximate Stackelberg equilibrium then implies that
    \begin{equation}
        \label{eqn:equil-soundness}
        \Ex[x \sim \mu]{ V^{A,B}(x) } \leq \frac{1-\gamma}{\rho^{d+1}} + \eta + \alpha.
    \end{equation}
    On the other hand, by \pref{lem:completeness} we have that there exists a strategy $A'\in\cA$ such that even for the best response $B^* \in \cB$ to $A'$, the expected reward is at least
    \begin{equation}
        \label{eqn:equil-completeness}
        \Ex[x \sim \mu]{V^{A',B^*,\horacle}(x)} \ge \frac{1}{\rho^{d+1}} - \frac{2\eps}{\rho^{d+1}}
    \end{equation}
    Then by the definition of an $\alpha$-approximate Stackelberg equilibrium, followed by an application of \pref{eqn:equil-completeness} and \pref{eqn:equil-soundness} we have
    \begin{align*}
        \alpha &\geq  \Ex[x \sim \mu]{V^{A',B^*,\horacle}(x)} - \Ex[x \sim \mu]{ V^{A,B}(x) }\\
            &> \frac{1}{\rho^{d+1}} - \frac{2\eps}{\rho^{d+1}} - \frac{1-\gamma}{\rho^{d+1}} - \eta - \alpha..
    \end{align*}
    Rearranging and using $\eta = \eps / \rho^{d+1}$, $\alpha < \frac{\epsilon}{\rho^{d+1}}$ completes the proof:
    \[
    \gamma \le 2{\eps} + \eta{\rho^{d+1}} + 2\alpha{\rho^{d+1}} < 5\eps.
    \]
\end{proof}
\section{Recommendations for empirical debate}
\label{sec:recommendations}

We make several qualitative recommendations for empirical debate based on our new protocol:
\paragraph{Asymmetry between the debaters.}
Much empirical work on debate thus far has had symmetric debaters, distinguished only by which debater speaks first. Our protocol suggests that it is useful to have asymmetry between the debaters. The debater $A$ should be tasked with proposing a solution and providing evidence to defend it, while debater $B$ should attempt to identify flaws, and evaluate the plausibility of the evidence.
This setup also makes sense for empirical debate research on open-ended tasks, as the first debater will need to propose a solution and then defend its correctness.

\paragraph{Explicit uncertainty estimates.}
Our protocols rely heavily on the ability to ask the estimator $B$ to produce probability estimates for subclaims that arise in the debate. The payoffs  are also determined by these estimates. In a practical setting one could use multiple samples from a generative model to produce uncertainty estimates, or directly use token probabilities of the answer from an LLM. For example, after an LLM $A$ generates an argument, one could prompt both $A$ and opposing LLM $B$ to read through the argument step-by-step, and at each step answer whether the claim seemed true. Token probabilities could then be compared between $A$ and $B$ for the truth of each claim.

\paragraph{Game theoretically motivated training algorithms.}
We depart from prior theoretical work on debate by considering different game theoretic equilibrium concepts to reason about our protocols.
The leader-follower structure implies that for every training update made to $A$, we should make multiple updates to $B$.
This ensures that whenever $A$ receives a gradient update, it comes from an episode of debate where $B$ is playing an approximate ``best response'' to $A$'s current strategy.
This type of training algorithm makes particular sense in the asymmetric setting where $A$ is proposing a complex solution to some problem, and is asked to defend the solution.
In this case, we would like to train $B$ to ensure the ability to identify flaws in $A$'s solutions, before providing a new  update to $A$.

\end{document}